\documentclass{article}

\PassOptionsToPackage{numbers}{natbib}
\usepackage[main, final]{neurips_2026}

\usepackage[utf8]{inputenc} %
\usepackage[T1]{fontenc}    %
\usepackage{hyperref}       %
\usepackage{url}            %
\usepackage{booktabs}       %
\usepackage{amsfonts}       %
\usepackage{nicefrac}       %
\usepackage{microtype}      %
\usepackage{xcolor}         %

\usepackage{microtype}
\usepackage{graphicx}
\usepackage{subcaption}
\usepackage{booktabs} %
\usepackage{placeins}
\usepackage{amsmath}
\usepackage{algorithm}
\usepackage{algpseudocode}
\usepackage{multirow}
\usepackage{amssymb}
\usepackage{mathtools}
\usepackage{amsthm}
\usepackage{wrapfig}
\usepackage{cleveref}
\title{Coupled Guidance for Flow Matching}

\author{%
    Mateo Cl\'emente \\
    Huawei \\
    \texttt{mateo.clem@gmail.com} \\
    \And
    Léo Maxime Brunswic\\
    Huawei\\
    \texttt{leo@brunswic.fr}\\
    \And
    Rui Heng Yang\\
    Huawei\\
    \texttt{rui.heng.yang@huawei.com}\\
    \And
    Amir Rasouli\\
    Huawei \\
    \texttt{amir.rasouli@huawei.com}\\
}

\newtheorem{theorem}{Theorem}

\def\pigdm{$\Pi$GDM }
\def\pigdmp{$\Pi$GDM$_{+}$}
\def\gsim{$g^{\text{sim-inv-A}}$ }

\begin{document}

\maketitle

\begin{abstract}
Guidance methods are essential tools for steering the generation process in diffusion models, improving the generalizability of stochastic robotic controllers and enabling high-quality controlled image synthesis. Despite strong parallels between flow-matching and diffusion, integrating guidance into flow-matching models remains quite challenging. Recent work has proposed general guidance frameworks for flow-matching, however, they are direct adaptations of diffusion guidance methods and can prove unstable. In this work, we conduct a theoretical analysis of guidance in flow-matching and show that a naive translation of diffusion-based guidance to flow-matching breaks the coupling distribution used during training, resulting in weaker control over sampling accuracy. Building on this insight, we propose \textbf{Coupled Guidance}, a flow-matching-specific guidance mechanism derived from a geometric perspective that leverages the coupling structure inherent to flow-matching training. To validate the effectiveness of our approach, we conduct extensive experimentation on a custom robotic collision avoidance task using \textsc{Maniskill3}, four robotic dexterous manipulation tasks from the \textsc{Adroit} benchmark, and three image generation tasks, namely {inpainting}, {super-resolution}, and {deblurring}, using the ImageNet-128 and CelebA-HQ-256 datasets. Our findings show that our framework improves sample quality and guidance stability across all tasks, without imposing any additional overhead.
\end{abstract}
\section{Introduction}

Flow-matching~\cite{lipman2022flow,liu2022flow} has recently emerged as a powerful class of generative models, unifying advantages from both normalizing flows~\cite{papamakarios2021normalizing,grathwohl2018ffjord} and diffusion models~\cite{Ho_2020_denoising_Neurips,karras2022elucidatingdesignspacediffusionbased}. The diffusion-inspired MSE loss in flow-matching removes the need for backpropagation through an ODE solver, improves numerical stability, and introduces a beneficial geometric bias~\cite{wan2025elucidating}. Meanwhile, Neural ODE parameterizations~\cite{chen2018neural} offer highly expressive continuous-time dynamics and efficient sampling. Owing to this combination of features, flow-matching achieves strong results across various domains of generative AI, including robotic control~\cite{Wang_2025_CVPR,black2025pi, Xing_2025_CVPR}, image generation~\cite{Stoica_2025_ICCV, Schusterbauer_2025_CVPR}, audio synthesis~\cite{Wang_2025_ICCV, Ki_2025_ICCV}, and forecasting~\cite{Fu_2025_CVPR,rohbeck2025modeling}.

As flow-matching continues to scale to more complex domains, mechanisms for \emph{controlling} its sampling process become increasingly important. Guidance has played this role with great success in diffusion models~\cite{ho2022classifier,bansal2023universal,karras2024guiding}, enabling interactive editing~\cite{epstein2023diffusion}, improved action consistency~\cite{black2025realtimeexecutionactionchunking}, and improved robustness in downstream tasks. In robotics, for example, diffusion-based controllers are often trained offline, hence may never encounter obstacles during training. Guidance allows the injection of collision-avoidance heuristics at inference time~\cite{ye2025ra,yang2025capecontextawarediffusionpolicy}, adapting to previously unseen environments. In general, guidance methods steer the generation by perturbing model parameters~\cite{rajabi2025token}, reweighting distributions via task objectives~\cite{Lu_ICML_2023}, or conditioning on auxiliary cues, such as class labels~\cite{dhariwal2021diffusion}. Despite their success in diffusion models~\cite{Li_2025_CVPR_2,Ye_2025_CVPR,ye2025ra,shen2024understanding,yang2025capecontextawarediffusionpolicy}, analogous techniques for flow-matching remain underdeveloped. Recently, a flow-matching guidance framework, obtained by adapting diffusion-based guidance mechanisms to flow-matching, demonstrated promising results in locomotion and image synthesis. However, these adaptations are direct translations of their diffusion counterparts and do not explicitly account for the training coupling that is intrinsic to flow-matching but absent in diffusion models.

\begin{wrapfigure}{t}{0.53\textwidth}
    \vspace{-10pt}
    \includegraphics[width=0.5\textwidth]{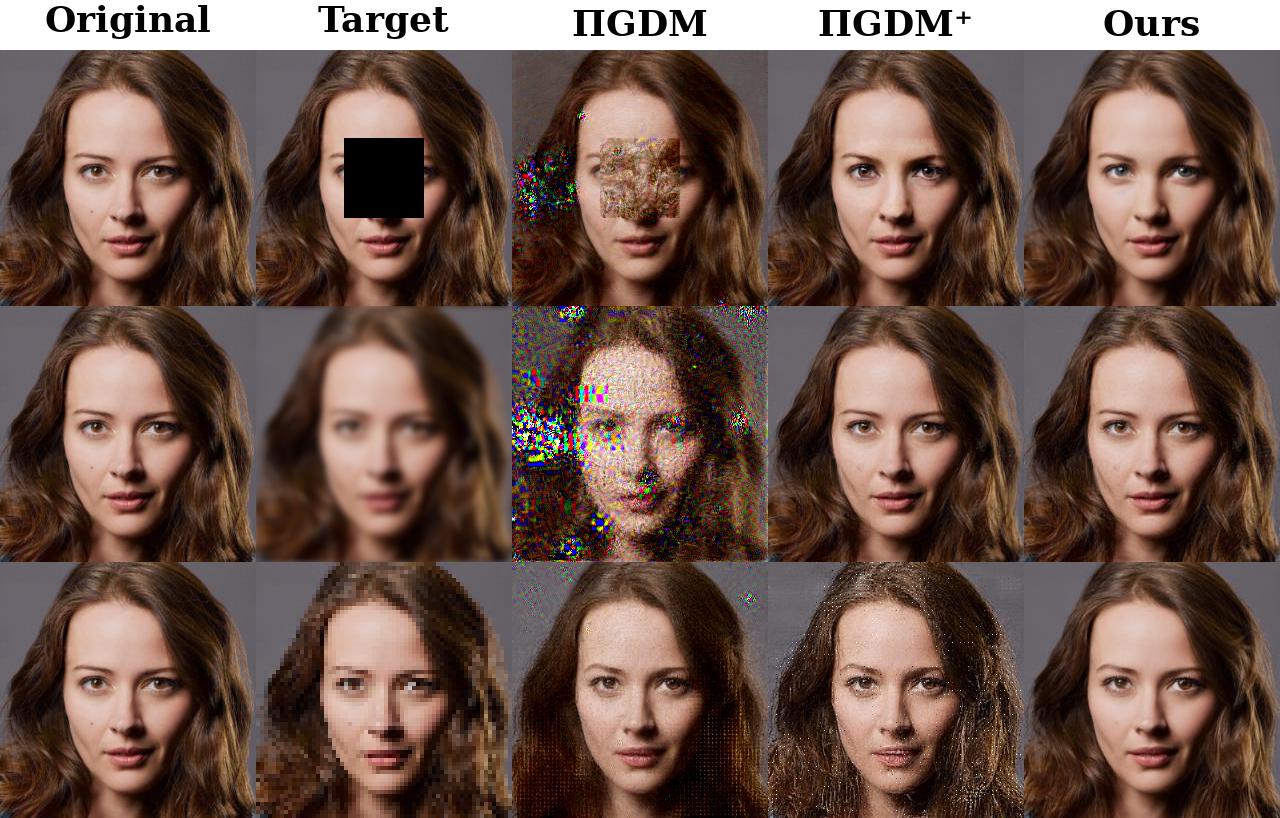}
    \caption{Qualitative examples of image reconstructions for noiseless inpainting, deblurring, and super-resolution, comparing naive \pigdm, \pigdmp, and our coupled guidance method.}
     \label{fig:page1_comparison} %
    \vspace{-10pt}
\end{wrapfigure}

In this work, we revisit training-free guidance for flow-matching from the perspective of the latent-data coupling on which flow-matching is built. This coupling, fixed during training, ensures that the model learns a consistent joint evolution of latent variables and data. We show that naive guidance violates this coupling, producing mismatched trajectories that hinder controllability and sampling accuracy. Leveraging this insight, we propose \emph{Coupled Guidance}, an efficient correction to the guided sampling process that preserves the training coupling. Our method employs a coordinated pair of guidance terms: a \emph{data-guidance} component acting in \textit{data space} and a corresponding \emph{latent-guidance} component acting in \textit{latent space}. Together, these terms maintain the structure of the flow-matching coupling, resulting in robust and controllable guided sampling with negligible additional overhead and minimal hyperparameter tuning. We evaluate the proposed Coupled Guidance method on robotic motion control with obstacle avoidance, four dexterous-hand manipulation tasks, and three image generation problems, namely inpainting, deblurring, and super-resolution (see Figure \ref{fig:page1_comparison}).

In sum, the main contributions of our work are as follows: i) A correction of guidance for flow-matching that jointly controls guidance in data space and in latent space; ii) A theoretical bound demonstrating the necessity of synchronizing guided latent and data distributions; iii) Extensive empirical evaluations on robotics and image generation tasks to validate the proposed guidance mechanism.

\section{Related Works}

\paragraph{Diffusion Models and Flow-Matching}
Diffusion models \cite{dhariwal2021diffusion,Yang_diffusion_2020} are a class of generative frameworks that operate by adding random noise to data samples to learn a reverse denoising process in order to generate new samples. Although diffusion models have been very effective in various generative tasks~\cite{ye2025ra,choi2024dice,epstein2023diffusion,huang2023make},  they suffer from multiple shortcomings, such as noise distribution assumptions and computationally expensive denoising processes. These have led to a number of recent works on solving these issues~\cite{salimans2022progressive,karras2022elucidatingdesignspacediffusionbased,zheng2023fast,clemente2025two}.

Flow-matching \cite{lipman2022flow} remedies these shortcomings by directly learning a continuous transport vector field, mapping the latent distribution to the data distribution. Unlike diffusion models, flow-matching enables the choice of any latent distributions (not necessarily Gaussian) and coupling, enhancing the fidelity of generative modeling and simplifying the transport plans \cite{liu2022flow}. As a result, high-quality samples can be achieved with very few sampling steps, as low as a single function evaluation \cite{frans2025stepdiffusionshortcutmodels}. Due to such characteristics, flow-matching has been adopted widely in various domains of AI, ranging from audio \cite{Wang_2025_ICCV, Ki_2025_ICCV} and image generation \cite{Stoica_2025_ICCV, Schusterbauer_2025_CVPR} to forecasting \cite{Fu_2025_CVPR,rohbeck2025modeling}, and robotic control and planning \cite{Wang_2025_CVPR,black2025pi, Xing_2025_CVPR}.

\paragraph{Guidance Methods}
In many generative tasks, it is essential to control the sampling process to meet certain criteria, such as generating new images given specific text prompts \cite{li2024tuning}. Guidance methods are effective mechanisms for controlling the sampling process and can be applied in different ways. 
One way is to apply guidance using a pretrained classifier to condition the sampling process \cite{dastider2024apex,Yang_diffusion_2020,zheng2025diffusionbased,Lin_2025_CVPR_guidance}. This approach, however, limits the diversity of samples as the conditioning is based on a limited set of classes. In addition, such a method requires large-scale retraining if the criteria are changed. For instance, in the context of obstacle avoidance, over 500K training samples are required to capture various goal configurations and obstacle layouts to cover a simple workspace \cite{dastider2024apex}.

Training-free guidance methods \cite{shen2024understanding} lower the computational burden by applying and optimizing a cost function directly at inference time, hence eliminating the need to retrain the generative model. These approaches have been widely applied to diffusion policies for robotic control applications in which responding to novel or unexpected situations is paramount \cite{carvalho2023motion,ye2025ra,feng2025language,pearce2023imitating,yang2025capecontextawarediffusionpolicy}. Despite the great success of training-free guidance methods, their direct application to flow-matching is not trivial. \citet{feng2025guidance} propose a comprehensive framework for the translation of diffusion guidance to flow-matching. This approach enables the adaptation of several guidance methods from diffusion to flow-matching and provides a strong baseline for our work. Nevertheless, these adaptations do not account for the training coupling. As we show in the following section, this omission can result in the velocity field being evaluated at points that lie far outside the coupling region, where the model has not been trained.

\paragraph{Base Methods}

Unlike in diffusion, the literature in pixel-space training-free guidance is still scarce for flow matching, making the number of relevant baselines limited. We chose to study two paradigms, namely Jacobian-based and Jacobian-free guidance by testing our method on \pigdm and $g^{\text{sim-inv-A}}$ \cite{feng2025guidance}, we refer to the latter as \pigdmp. 

\pigdm is a direct adaptation of the pseudoinverse-based guidance  \cite{song2023pseudoinverse} to the flow-matching setting. In noiseless inverse problems, the guidance takes the form,
\[
\Pi \text{GDM}(x,y,H) = - \lambda_t ((H^\dagger y - H^\dagger H\hat x_1)^T\frac{\partial \hat x_1}{\partial x_t})^T
\]
where $\lambda_t$ denotes the guidance schedule, $H$ is the degradation operator (e.g., masking for inpainting, blurring for deblurring, or downsampling for super-resolution), and $y = H x^*$ is the degraded observation of the target image $x^*$. The operator $H^\dagger = H^\top (H H^\top)^{-1}$ corresponds to the Moore–Penrose pseudoinverse of $H$. \pigdmp follows that exact formulation but approximates $\frac{\partial \hat x_1}{\partial x_t}$ as the identity.

Finally, since no trivial adjoint operator $H^T$ is available for collision avoidance, we instead apply guidance by directly using $g^{cov-A}$ which minimizes an energy function, $$g_{\text{collision}}=-\lambda_t\nabla_{\widehat x_1}||\widehat x_1-y||$$
with $y$ the target generated by our collision avoidance pipeline. We provide a discussion about our choice of baseline in Appendix \ref{appendix:baselines}.

\section{A Short Theory of Coupled Guidance}

\subsection{Notations}
We consider a data space $\mathcal X := \mathbb R^d$, viewed as an affine space with underlying vector space $\overrightarrow{\mathcal X}$. It is equipped with its Borel $\sigma$-algebra and the Lebesgue measure~$\lambda$. We denote measures by Greek letters $\mu,\nu,\pi$. A \emph{distribution} is a probability measure on $\mathcal X$, and a \emph{coupling} of two distributions $\mu$ and~$\nu$ is a distribution $\pi$ on $\mathcal X\times\mathcal X$ such that its marginals are $\mu$ and $\nu$ respectively.

\subsection{Flow-Matching Training and Inference}
A flow-matching model consists of a tuple $(\mu_1,\mu_0,\pi,v^\theta)$ where $\mu_1$ is the data or target distribution, $\mu_0$ is a latent distribution (typically Gaussian), $\pi$ is a coupling between them, and $v^\theta : \mathcal X\times[0,1]\to\overrightarrow{\mathcal X}$ is a learnable vector~field.

The model is trained using a linear flow-matching objective,
$$\mathcal L(\theta)
    = \mathbb E_{(x_0,x_1)\sim\pi} \int_0^1 
       \| v^\theta_t(x_t) - v_t(x_t) \|^2\, dt  
    = \mathbb E \int_0^1 
       \| v^\theta_t(x_t) - (x_1 - x_0)\|^2 dt + \mathrm{Cte}, \nonumber
$$
where $x_t = (1-t)x_0 + t x_1$ and $v_t(x_t) = \mathbb E[x_1 - x_0 \mid x_t]$.

At inference time, one solves the ODE $\dot x_t = v^\theta(x_t,t)$, typically via a first-order Euler scheme:
\begin{equation}\label{eq:update_euler_scheme}
    x_{t+\delta t} = x_t + v^\theta_t(x_t)\,\delta t,\quad
\end{equation}
We can extend this Euler step for a one step prediction:
\begin{equation}\label{eq:x1_pred}
    \widehat x_1 = x_t + (1-t)v^\theta(x_t,t)
\end{equation}

\subsection{From Guided Flow to Coupled Guidance}
In practice, guidance takes the form of a task-specific function $\mathcal G : \mathcal X  \to \overrightarrow{\mathcal X},$
which acts as a correction to the predicted clean sample. For clarity, we omit conditioning variables (e.g., observation~$o$).

By analogy with universal diffusion guidance~\cite{bansal2023universal}, each inference step may be interpreted as (i) predicting a denoised endpoint $\widehat x_1$, followed by (ii) interpolating toward it,
\begin{align}
\label{eq:update_naive_interpolation}
x_{t+\delta t} &= \frac{1-t-\delta t}{1-t}\,x_t
               + \frac{\delta t}{1-t}\,\widehat x_1.
\end{align}
Naive guidance first modifies the predicted endpoint,
\begin{align}
\widehat x_1^g &= \widehat x_1 + \alpha_t \mathcal G(\widehat x_1), \\
x_{t+\delta t}^{g,\mathrm{naive}}
    &= \frac{1-t-\delta t}{1-t}x_t
     + \frac{\delta t}{1-t}\,\widehat x_1^g,
\label{eq:update_naive_guidance}
\end{align}
but this formulation obscures the underlying coupling structure on which flow-matching is trained. We rewrite inference as a coupling update by defining the reconstructed endpoints as

\begin{equation}
\label{eq:x0hat}
\widehat x_0 = x_t - t\,v^\theta(x_t,t)\quad \text{and}  \quad \widehat x_1 = x_t + (1-t)v^\theta(x_t,t)
\end{equation}

Then the Euler update \eqref{eq:update_euler_scheme} is equivalent to
\begin{equation}\label{eq:update_coupling}
    x_{t+\delta t} = (1-t-\delta t)\widehat x_0 + (t+\delta t)\widehat x_1,
\end{equation}
a linear interpolation between the latent and target endpoints. When applying naive guidance, the target endpoint becomes $\widehat x_1^g$, and the latent endpoint must be modified to
\[
\widehat x_0^{g,\mathrm{naive}}
    := x_t - t\,\frac{\widehat x_1^g - x_t}{1-t}.
\]
Thus naive guidance implicitly induces a \emph{latent guidance}
\[
\widehat x_0^{g,\mathrm{naive}} - \widehat x_0
    = -\frac{\alpha_t t}{1-t}\,\mathcal G(\widehat x_1).
\]

\paragraph{General coupled guidance.}
Motivated by this observation, we introduce \textbf{Coupled Guidance}, the pair $(\mathcal G^0,\mathcal G^1)$ and the following update rules in the general form:
\begin{align}
\widehat x_0^g &= \widehat x_0 + \mathcal G^0(\widehat x_0,\widehat x_1,t), 
\label{eq:latent_guide}\\
\widehat x_1^g &= \widehat x_1 + \mathcal G^1(\widehat x_0,\widehat x_1,t), 
\label{eq:target_guide}\\
x_{t+\delta t}^g
    &= (1-t-\delta t)\widehat x_0^g
     + (t+\delta t)\widehat x_1^g.
\label{eq:update_coupling_guidance_coguidance}
\end{align}

For naive guidance, we have

\begin{equation}
\mathcal G^{1,\mathrm{naive}} = \alpha_t \mathcal G(\widehat x_1) \quad \text{and} \quad
\mathcal G^{0,\mathrm{naive}} = -\frac{\alpha_t t}{1-t}\mathcal G(\widehat x_1).
\end{equation}

\subsection{Training Coupling versus Guided Coupling}

A coupled guidance induces a guided interpolant distribution
$
(\widehat x_0^g,\widehat x_1^g) \sim (\pi^g \mid o,t),
$
which depends on the observation~$o$ and timestep~$t$. We continue by characterizing how deviations between $\pi^g$ and the training coupling~$\pi$ affect flow-matching accuracy, and clarify why naive guidance degrades performance even if  $\Pi$GDM mitigates some—but not all—failures.

\paragraph{Assumptions.}
Following~\citet{benton2023error}, we assume:
\begin{itemize}
    \item the flows induced by $v_t$ and $v_t^\theta$ exist, and are unique and smooth;
    \item $v^\theta$ is differentiable in $(t,x)$ and Lipschitz in $x$;
    \item the guided update admits a continuous-time limit
          $x_{t+\delta t}^g = x_t^g + \delta t\,v_t^g + o(\delta t)$;
    \item $v_t \mapsto v_t^g$ is Lipschitz with constant $L_g$.
\end{itemize}
These hold, in particular, for naive guidance provided 
$\alpha_t\|\mathcal G\| = O(1-t)$ as $t\to 1$.
We denote by $\mu_1^{g(o)}$ and $\mu_1^{\theta,g(o)}$ the final guided distribution for the theoretical minimizer $v_t$ of the loss and the trained vector field $v_t^\theta$, respectively.

\begin{theorem}\label{theo:ood_bound}
Let $(\mu_1,\mu_0,\pi,v^\theta)$ be a flow-matching model and let $g$ be any coupled guidance. Under the above assumptions,
\begin{equation}
\mathbb E_o\,
\mathcal W\!\left(
    \mu_1^{g(o)},\,
    \mu_1^{\theta,g(o)}
\right)
\le
L_g L_\theta 
\,\mathcal L^{1/2}(\theta)\,
\biggl\|
    \frac{d\pi^g}{d\pi}
\biggr\|_{L^\infty(\pi)},
\end{equation}
where $\pi^g$ is the guided coupling distribution,  
$L_\theta = \exp\!\left(\int_0^1 L_{\theta,t}dt\right)$,  
and $L_{\theta,t}$ is the Lipschitz constant of $v_t^\theta$; See Appendix~\ref{appendix:Th1_proof} for the proof.
\end{theorem}

\paragraph{Interpretation.}
The bound highlights the central difficulty of naive guidance:  
 the factor $\frac{d\pi^g}{d\pi}$ measures how far the guided coupling deviates from the training coupling. Under the simplifying assumptions that the couplings are independent, $\pi=\mu_0 \otimes \mu_1$ and $\pi^g = \pi^g_0\otimes \pi_1^g$, the density ratio decomposes into
\[
\frac{d\pi^g}{d\pi}
\quad=\quad
\frac{d\pi_1^g}{d\mu_1}
\cdot
\frac{d\pi_0^g}{d\mu_0}.
\]

The target marginal $\pi_1^g$ depends on how well the \emph{target guidance} $\mathcal G^1$ preserves the data manifold.  
Under the manifold hypothesis~\cite{fefferman2016testing}, $\mu_1$ is highly anisotropic;  
$\Pi$GDM improves stability by projecting the guidance direction onto dominant tangent directions of~$\mu_1$, thus controlling the data-side marginal density $\frac{d\pi_1^g}{d\mu_1}$. However, computing this projection requires $\frac{d\hat x_1}{dx_t}$, which can become ill-behaved when guided trajectories deviate from the training coupling. Using coupled guidance controls the latent-side density, keeping the inference stable.

The latent marginal $\pi_0^g$ depends on the \emph{latent guidance} $\mathcal G^0$.  
For a Gaussian prior, naive guidance yields
\[
\frac{d\pi_0^g}{d\mu_0}
\approx
\mathbb E_{x_1\sim\mu_1}
\!\int_{0}^1\!
\exp\!\Bigl(
    \bigl\|\tfrac{\alpha_t}{1-t}\mathcal G(x_1)\bigr\|^2
\Bigr)\,dt,
\]
which can explode for moderate guidance strength, causing severe out-of-distribution drift of intermediate $x_t^g$.

\paragraph{Conclusion.}
Naive guidance evaluates the learned vector field $v^\theta$ at intermediate points $x_t^g$ that may lie far outside the coupling used during training, leading to large density ratios and degraded inference.  
This motivates designing guidance rules that preserve the latent--data coupling—precisely the principle behind our Coupled Guidance framework.

\section{Method} \label{sec:method}

\subsection{Coupled Guidance Algorithm}
Our method directly instantiates the coupled guidance update rules 
\eqref{eq:latent_guide}, \eqref{eq:target_guide}, 
\eqref{eq:update_coupling_guidance_coguidance}.  
We use the same target-guidance as standard guided flow, but modify the latent-guidance term to better preserve the training coupling. Concretely, we define,
\begin{align}
\mathcal G^{1,\alpha}(\widehat x_0, \widehat x_1, t) 
    &= \alpha\, \mathcal G(\widehat x_1), \\
\mathcal G^{0,\beta}(\widehat x_0, \widehat x_1, t) 
    &= -\frac{\beta t}{1-t}\, \mathcal G(\widehat x_1),
\end{align}
where $(\alpha,\beta)$ are hyperparameters controlling data-side and latent-side guidance.  
Algorithm~\ref{alg:method} summarizes the full sampling procedure.

This procedure introduces only two hyperparameters and retains the same computational complexity as naively guided flow-matching. Note that $\alpha=1$ will almost always be the sensible choice, as scaling of data-guidance is in practice already included in the guidance operator itself. Only $\beta$ needs to be chosen between naive ($\beta=1$) and coupled guidance ($\beta \in [-1,0]$). We see very little empirical benefit to tuning $\beta$ carefully, and would recommend simply setting $\beta \in \{0, -1\}$ as a starting point.

Intuitively, $\beta$ controls both the magnitude and sign of the latent correction relative to the data-side guidance $\mathcal{G}$, acting as an affine parameter along ($\hat{x}_0$,$\hat{x}_0^g$), with a clear geometric interpretation:

\begin{itemize}
    \item 
$\beta = 1$: the latent interpolant is $\hat{x}_0^g$, the correction is reversed relative to the data guidance, actively opposing it.
    \item 
$\beta = 0$: the latent interpolant remains $\hat{x}_0$ (no latent correction).
    \item 
$\beta = -1$: the latent interpolant is set to the symmetric of $\hat{x}_0^g$ with respect to $\hat{x}_0$, i.e. it is pushed in the same direction as the data guidance (since $-\frac{-1}{1-t}\mathcal{G} = +\frac{1}{1-t}\mathcal{G}$).
\end{itemize}

As established in \citet[Theorem 3.5]{liu2022flow}, flow-matching training minimizes transport cost, which causes the learned vector field to progressively align latent and data interpolants along trajectories. This means the learned data-latent coupling tends to co-orient $\hat{x}_0$ and $\hat{x}_1$ along the flow. Applying $\beta = 1$ pushes the latent interpolant against the data guidance direction, producing a data-latent coupling that diverges from the learned one. By contrast, $\beta \leq 0$ aligns the latent correction with the data guidance, keeping the guided coupling closer to the learned coupling.

\begin{figure}[t]
  \begin{minipage}[t]{0.48\textwidth}
    \vspace{-1em}
    \begin{algorithm}[H]
        \footnotesize
        \caption{Coupled Guidance Sampling}
        \label{alg:method}
        \begin{algorithmic}[1]
        \Require Number of steps $T$, velocity field $v^\theta_t(x)$, guidance function $\mathcal{G}$, coefficients $(\alpha,\beta)$, step size $\delta t$.
        \State Sample initial latent $x_0^g \sim \mu_0$ and set $t \gets 0$
        \For{$i = 0..T$}
            \State $v_t \gets v^\theta(x_t^g, t)$
            \State $\widehat{x}_1 \gets x_t^g + (1-t)v_t$
            \State $\widehat{x}_0 \gets x_t^g - t v_t$
            \State $g \gets \mathcal G(\widehat{x}_1)$
            \State $\widehat{x}_1^g \gets \widehat{x}_1 + \alpha g$
            \State $\widehat{x}_0^g \gets \widehat{x}_0 - \frac{\beta t}{1-t} g$
            \State $t \gets t + \delta t$
            \State $x_t^g \gets (1-t)\widehat{x}_0^g + t \widehat{x}_1^g$
        \EndFor
        \State \Return $x_t^g$
        \vspace{0.3em}
        
        \end{algorithmic}
    \end{algorithm}
  \end{minipage}
  \hfill
  \begin{minipage}[t]{0.48\textwidth}
        \centering
        \captionof{table}{Comparison of naive vs. coupled guidance across methods and tasks on ImageNet-128 (minibatch OT).}
        \label{tab:results_imagenet}
        \resizebox{\linewidth}{!}{%
        \begin{tabular}{llcccc}
        \toprule
         & & \multicolumn{2}{c}{\textbf{Naive}} & \multicolumn{2}{c}{\textbf{Coupled}} \\
        \cmidrule(lr){3-4} \cmidrule(lr){5-6}
        \textbf{Task} & \textbf{Metric} & $\Pi$GDM & $\Pi$GDM+ & $\Pi$GDM & $\Pi$GDM+ \\
        \midrule
        \multirow{4}{*}{Inpainting}
          & FID$\downarrow$   & 18.3  & 10.3  & \textbf{8.5}   & 11.4 \\
          & LPIPS$\downarrow$ & 0.112 & \textbf{0.086} & \textbf{0.088} & 0.089 \\
          & PSNR$\uparrow$    & 18.3  & 20.6  & \textbf{23.5}  & 20.3 \\
          & SSIM$\uparrow$    & 0.903 & 0.911 & \textbf{0.916} & 0.908 \\
        \midrule
        \multirow{4}{*}{Deblurring}
          & FID$\downarrow$   & 70.5  & 14.3  & 22.2  & \textbf{13.2} \\
          & LPIPS$\downarrow$ & 0.471 & 0.197 & 0.264 & \textbf{0.184} \\
          & PSNR$\uparrow$    & 19.0  & 23.8  & 23.5  & \textbf{24.3} \\
          & SSIM$\uparrow$    & 0.449 & 0.753 & 0.711 & \textbf{0.769} \\
        \midrule
        \multirow{4}{*}{Super-res.}
          & FID$\downarrow$   & 49.7  & 46.0  & \textbf{8.5}   & 50.4 \\
          & LPIPS$\downarrow$ & 0.335 & 0.396 & \textbf{0.142} & 0.413 \\
          & PSNR$\uparrow$    & 19.0  & 17.5  & \textbf{24.0}  & 17.3 \\
          & SSIM$\uparrow$    & 0.645 & 0.477 & \textbf{0.816} & 0.452 \\
        \bottomrule
        \end{tabular}%
        }
  \end{minipage}
  \vspace{-15pt}
\end{figure}
\vspace{5pt}
\section{Experiments}\label{sec:experiments}

\subsection{Collision Avoidance}

Following \citet{chi2025diffusion,clemente2025two}, we use a 65M-parameter U-Net \cite{ronneberger2015u} to parameterize our flow-matching policy, trained on the ManiSkill3 \cite{tao2024maniskill3} Pick Cube environment and its associated dataset. The action space is 9-dimensional, corresponding to the robot joint positions. 
\begin{wraptable}{l}{0.26\linewidth}
  \centering 
  \caption{Collision avoidance success rates (\%) for all $(\alpha,\beta)$ couples, over 256 environment seeds.}
  \label{tab:coll_avoidance}
  \footnotesize
  \begin{tabular}{lccc}
  \toprule
  \textbf{$\alpha\backslash\beta$} & -1 & 0 & 1 \\
  \midrule
  0 & 73 & 21 & 17 \\
  1 & \textbf{81} & 77 & 36 \\
  \bottomrule
  \end{tabular}
\end{wraptable}
The model is trained to predict an action horizon of length $h_{\mathcal A}=24$—resulting in an output dimension of $24\times9=216$—conditioned on an observation horizon of $h_{\mathcal O}=2$ past environment states. On the base environment, the learned policy solves the task almost perfectly, achieving a success rate of 99.8\% (see Appendix \ref{appendix:robot_experiments} for more details).
We then evaluate the same policy in a modified environment where an obstacle is introduced between the initial cube position and the target, and apply guidance to enable collision avoidance. The guidance function is defined using a Euclidean signed distance field (ESDF), computed with cuRobo \cite{sundaralingam2023curoboparallelizedcollisionfreeminimumjerk}, and acts directly on the predicted joint trajectories to penalize collisions. This form of guidance is inherently noisy, as it perturbs each joint state independently across the prediction horizon, resulting in a large deviation between the guided and training couplings, reflected by a high density ratio $d\pi_1^g / d\mu_1$.

Without guidance ($\alpha=0$, $\beta=0$), the policy’s performance drops to approximately 21\%, as trajectories frequently collide with the obstacle. In this setting, \pigdm is not directly applicable, as it requires access to an adjoint or degradation operator, which is unavailable for this task. We therefore apply guidance by directly minimizing the ESDF-based energy. Under naive guidance ($\alpha=1$, $\beta=1$), performance only improves up to 36\%. In contrast, coupled guidance consistently yields strong performance despite the noisy nature of the guidance signal, achieving up to 81\% success rate with $\alpha=1$ and $\beta=-1$. Results for different tested $(\alpha,\beta)$ configurations are reported in Table~\ref{tab:coll_avoidance}.

\subsection{Robotic Control as an Inpainting Task}

\begin{figure*}[t]
    \centering
    \includegraphics[width=0.19\linewidth]{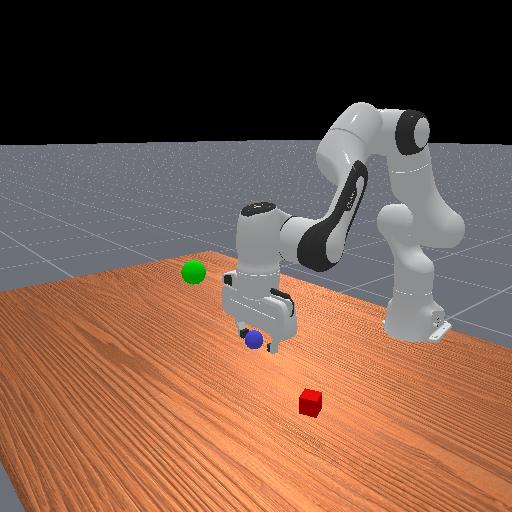}
    \includegraphics[width=0.19\linewidth]{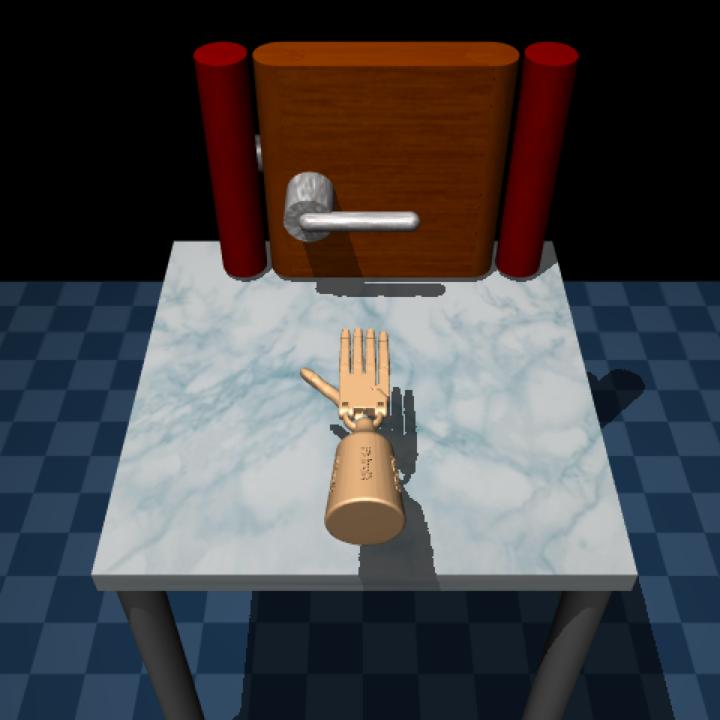}
    \includegraphics[width=0.19\linewidth]{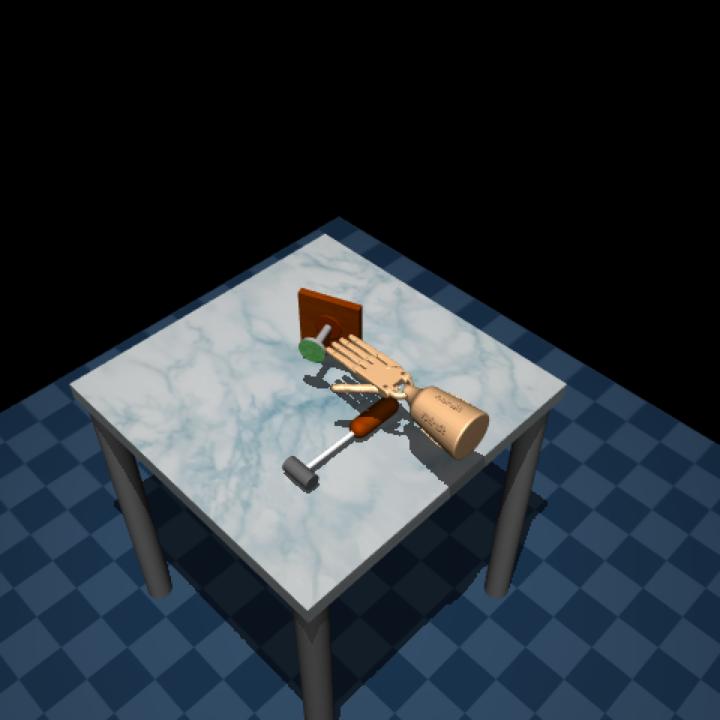}
    \includegraphics[width=0.19\linewidth]{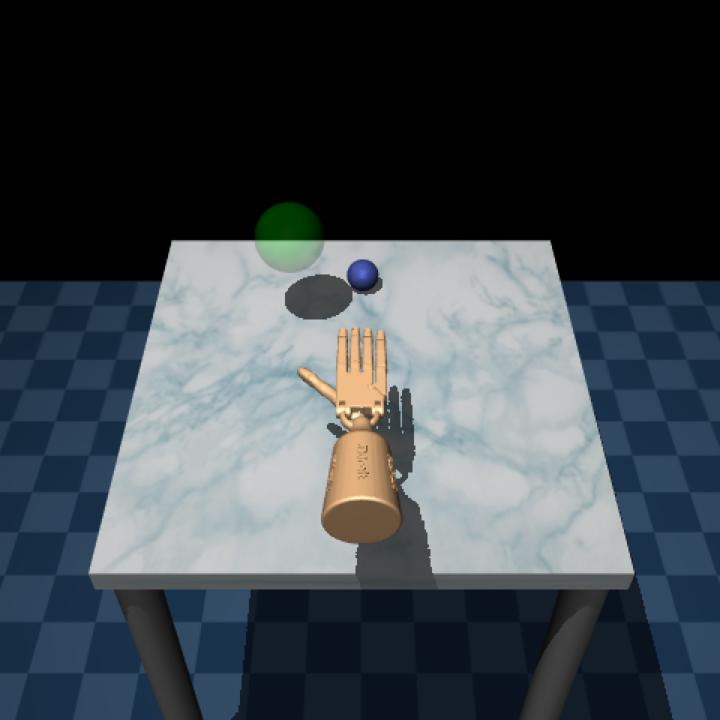}
    \includegraphics[width=0.19\linewidth]{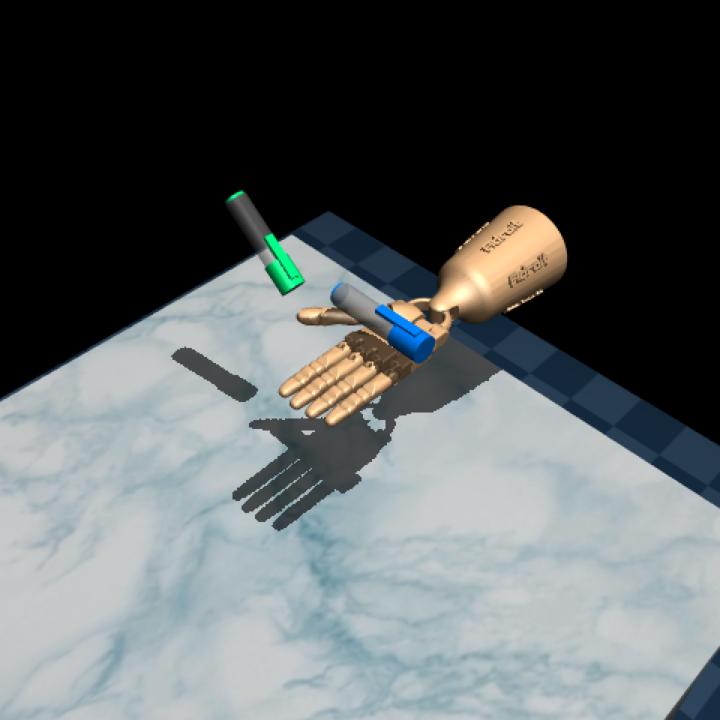}
    \caption{Experimental setups for robotics tasks. From left to right: Maniskill PickCube, Adroit Door, Hammer, Relocate, and Pen.}  %
    \label{fig:robotics_images}
    \vspace{-15pt}
\end{figure*}

We emulate a two-level hierarchical motion controller in which a high-level planner provides partial information about future trajectories to a low-level controller, thereby turning control into an inpainting problem. To this end, we train two flow-matching models per task. The high-level planner is a \emph{conditioned} model that, given an observation history, predicts a horizon of $h=24$ future state–action pairs $(o_t,a_t)_{t=1}^h$. The low-level controller is trained on the same dataset but remains \emph{unconditioned}.

At test time, the controller receives sparse supervision from the planner and its task is to infer a complete action sequence consistent with this partial trajectory information. Intuitively, the planner acts as a world-aware module that provides intermittent \textit{waypoints} in the trajectory space, while the controller acts as an action expert that fills in the missing degrees of freedom.

We evaluate robustness to guidance sparsity under both naive and coupled guidance on the four Adroit dexterous manipulation tasks \cite{rajeswaran2017learning, fu_2020_adroit}: Door (latch opening), Pen (reorientation to a target angle), Hammer (nail hammering), and Relocate (object relocation) (Figure~\ref{fig:robotics_images}). Different seeds correspond to variations in task configuration (e.g., door pose, target angle, nail position, and object/goal locations). As \pigdm and \pigdmp behave similarly, we present results for \pigdm in Figure~\ref{fig:sparsity} and include additional results in Appendix~\ref{appendix:robot_results}. For all runs we used 300 seeds.

As shown in Figure \ref{fig:sparsity}, coupled guidance is substantially more robust to sparse guidance than naive guidance, without requiring task-specific tuning. Even when sweeping a wide range of guidance scales, naive guidance degrades sharply once 20\% or fewer dimensions are guided, whereas coupled guidance maintains near-fully-guided performance down to approximately 5\% guided dimensions. Naive guidance can perform well with careful tuning, but requires different guidance scales across sparsity regions, whereas coupled guidance remains stable across all sparsity levels without tuning. Additional results  are provided in Appendix~\ref{appendix:robot_results}.

\begin{wrapfigure}{r}{0.72\linewidth}
  \vspace{-5pt}
  \centering
  \includegraphics[width=0.49\linewidth]{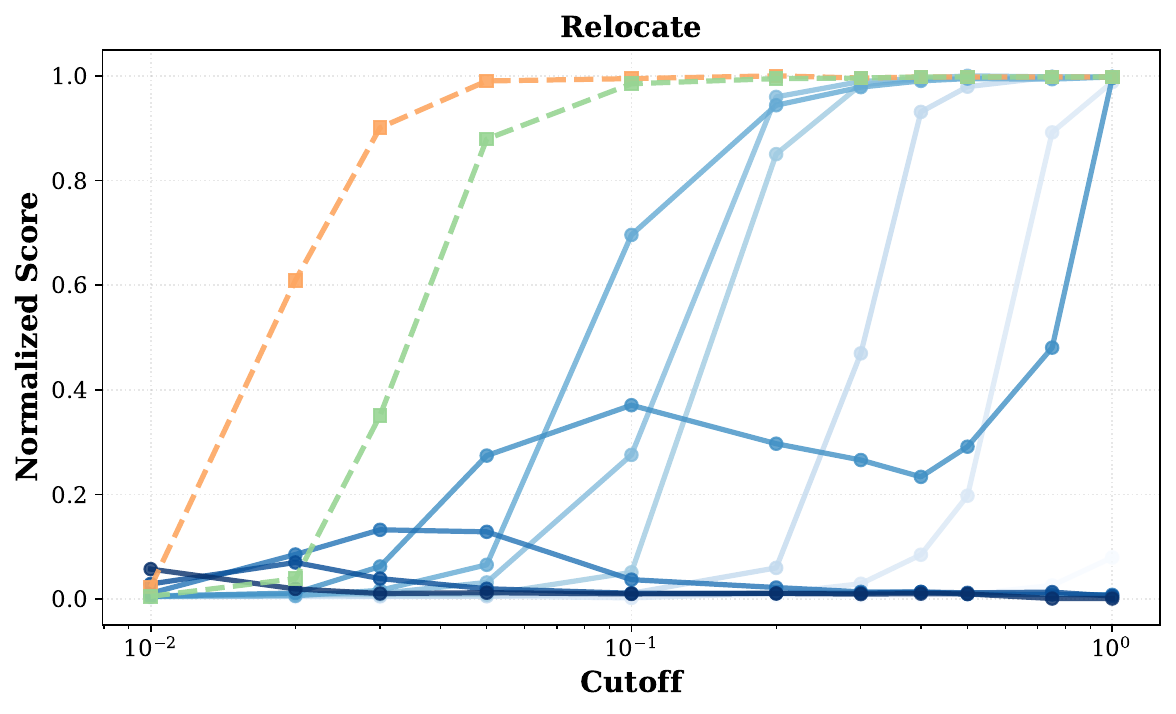}
  \includegraphics[width=0.49\linewidth]{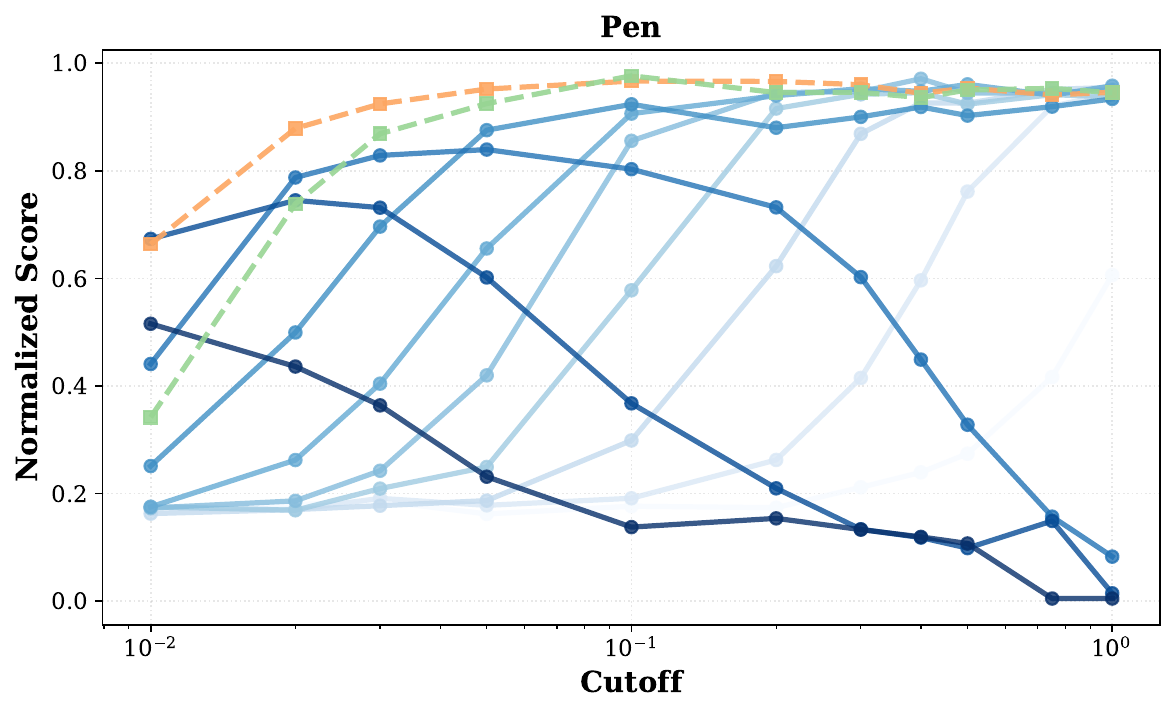}
  \includegraphics[width=\linewidth]{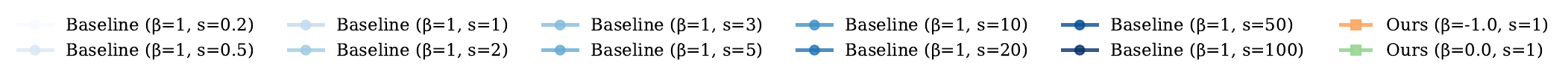}
  \caption{Performance of \pigdm on Adroit tasks given different sparsity cutoffs. Runs using coupled guidance are used directly without a unit guidance scale $s=1$. The blue lines each represent the performance of naive guidance for different guidance scalings $s$.}
  \label{fig:sparsity}
  \vspace{-9pt}
\end{wrapfigure}

Since inpainting-based guidance has recently been adopted to improve trajectory consistency in generalist robotic policies \cite{black2025realtimeexecutionactionchunking}, methods such as \pigdm\ have clear relevance to real-world robotic control. The increased stability provided by coupled guidance directly benefits these applications by reducing sensitivity to sparse or imperfect guidance, enabling more reliable deployment without overhead.

\subsection{Image Generation}
\begin{wrapfigure}{r}{0.4\textwidth}
    \vspace{-20pt}
    \centering
    \includegraphics[width=0.9\linewidth]{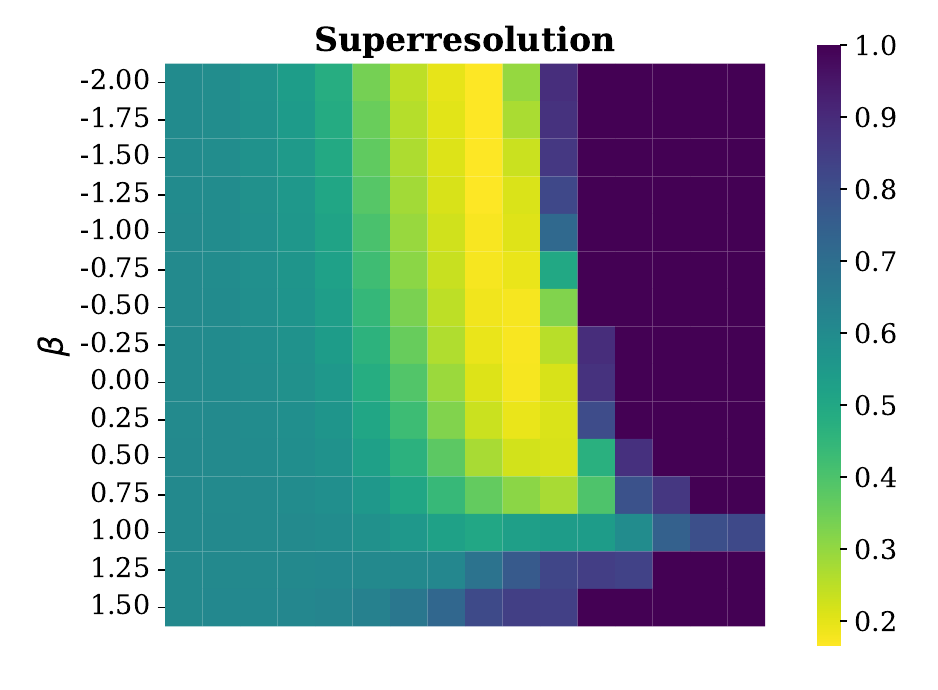}
    \includegraphics[width=0.9\linewidth]{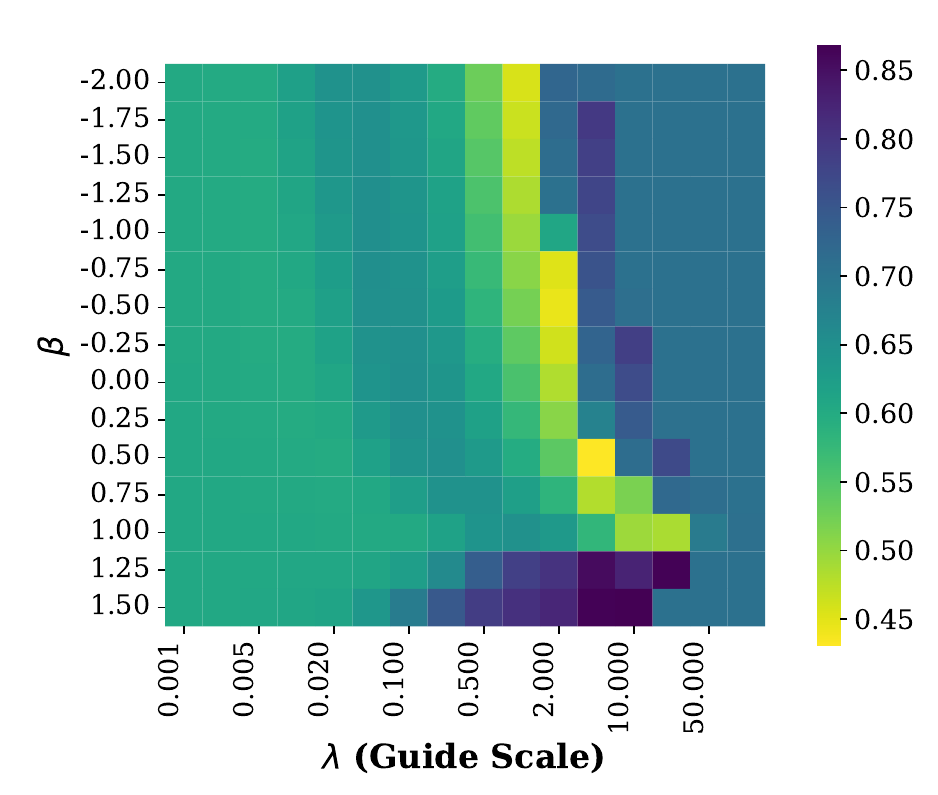}
            \caption{LPIPS scores for \pigdm, (first row) and \pigdmp (second row) on super-resolution}
    \label{fig:heatmaps}
\end{wrapfigure}
For the image generation tasks, we study the behavior of \pigdm and \pigdmp under naive and coupled guidance on box inpainting, deblurring and factor 4 super-resolution. We train a U-Net model from scratch on the CelebA-HQ-256 and ImageNet-128 datasets \cite{liu2015faceattributes}  (see Appendix \ref{appendix:image_experiments} for more details). All runs use the same base flow-matching model and are compared via 4 image quality metrics: FID \cite{heusel2017gans}, LPIPS \cite{zhang2018unreasonable}, PSNR and SSIM \cite{wang2004image}. To cover usual settings, we also train a minibatch OT model with a similar backbone for ImageNet-128. Given that the results are consistent for both datasets and coupling choices, we make general comments in the following section for conciseness. The full tables and plots of the results are provided in Table \ref{tab:results_imagenet} and Appendix \ref{appendix:images}.

 We evaluate our method on noiseless linear inverse problems:
 $$
y = Hx + \sigma_y, \quad \text{with} \quad \sigma_y = 0.
 $$
 
Although adding noise may make an accurate reconstruction more difficult, in practice it helps stabilize the guidance process, and $\sigma_y = 0$ is actually the regime in which training-free guidance is most prone to instability. It has been shown in the previous studies \cite{xia2026noiseadaptivediffusionsamplinginverse,bellchambers2025exploitingexactdenoisingposterior,tian2026stabilizingdiffusionposteriorsampling} that nonzero measurement noise $\sigma_y$ acts as an implicit regularizer, i.e. it smooths the posterior score so that the linear guidance term dominates for most of the trajectory \cite{bellchambers2025exploitingexactdenoisingposterior}, reduces the effective curvature of the likelihood gradient \cite{tian2026stabilizingdiffusionposteriorsampling}, and in \pigdm directly Tikhonov-regularizes the covariance inversion via the $\sigma_y^2$ I term. When $\sigma_y$ approaches 0, these stabilizing effects vanish: the guidance sharpens, instability increases, and approximation errors in the score are no longer masked by a forgiving posterior geometry. This makes the noiseless setting pertinent for our study as it highlights a failure mode where the stabilizing effects of coupled guidance are necessary.

\paragraph{Inpainting} 

In this setting,  naive \pigdmp exhibits stable and reliable performance across all metrics, whereas \pigdm shows worse performance, producing visible artifacts. This behavior, illustrated in Figure \ref{fig:heatmaps} is consistent with the analysis in Section~3, as \pigdm\ relies on the Jacobian $\partial \widehat{x}_1 / \partial x_t$, which becomes poorly conditioned when guided trajectories deviate from the training coupling. Coupled guidance has little effect on \pigdmp, whose stability primarily stems from the identity Jacobian approximation. However, it significantly stabilizes \pigdm by preserving the training coupling. As a result, coupled \pigdm consistently outperforms naive \pigdm and \pigdmp across all image quality metrics.

\paragraph{Super-resolution}

In this sparse guidance setting, both \pigdmp and \pigdm fail to produce meaningful reconstructions, yielding poor image quality across all evaluation metrics, as shown in Figure \ref{fig:heatmaps}. Like in the inpainting case, introducing coupled guidance has negligible impact on \pigdmp, whose performance is fundamentally constrained by the identity Jacobian approximation and thus poorly suited to this task. In contrast, coupled guidance substantially stabilizes \pigdm by mitigating excessive latent-side drift, enabling reliable sampling throughout the trajectory. This stabilization translates into a clear improvement in sample quality, with the coupled \pigdm variant consistently outperforming all other methods.

\paragraph{Deblurring} 

In this task, both \pigdmp\ and \pigdm\ benefit from coupled guidance. \pigdmp by default produces high-quality samples under naive guidance, indicating that the identity Jacobian approximation is well suited to this type of degradation. Still, coupled guidance further improves the reconstruction quality across all metrics in this setting. Coupled guidance likewise stabilizes \pigdm and enables reliable inference; however, even with this stabilization, its performance remains below that of \pigdmp. Overall, both methods benefit from coupled guidance in this setting even if this time \pigdmp\ emerges as the strongest method.  Additional illustrations for image quality metrics are provided in Appendix \ref{appendix:images}.

\begin{wrapfigure}{r}{0.6\textwidth}
    \centering
    \vspace{-0.5cm}
    \includegraphics[width=0.56\textwidth]{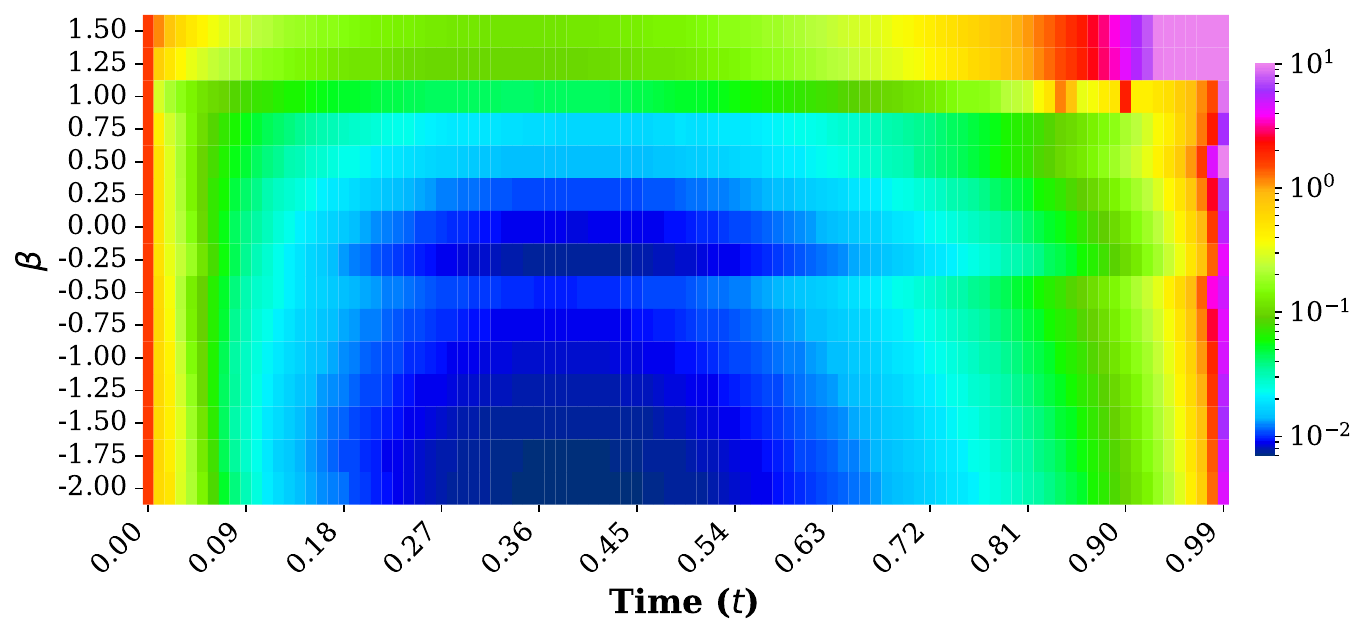}
    \caption{Evolution of $\partial \widehat x_1 /\partial x_t$'s contribution to the guidance norm for the CelebA-HQ-256 superresolution task through the \pigdm inference process.}
    \label{fig:divergence}
    \vspace{-0.3cm}
\end{wrapfigure}
\paragraph{Jacobian Influence} Figure~\ref{fig:divergence} illustrates how coupled guidance affects the contribution of the Jacobian-dependent term throughout the inference process. Note that to remain fair to the naive setting, we plot the best run for each $\beta$ after tuning. The term is initially large during the early steps, reflecting the expected role of the Jacobian in aligning the prediction with the inverse constraint. Crucially, when $\beta=1$ (naive guidance), the Jacobian contribution remains substantially larger in the stable zone, and the final destabilization occurs earlier,  making the final divergence more pronounced. Decreasing $\beta$ effectively dampens this amplification, extending the stable regime and delaying the time at which the term starts being ill-behaved. This observation provides direct empirical evidence that coupled guidance helps control the magnitude and timing of the Jacobian-induced corrections by preserving the training coupling, thereby preventing instability without degrading the guidance signal.

\section{Limitations}
\paragraph{Simplified theoretical setting.}
Our analysis is conducted under a simplified formulation using the unweighted quadratic flow-matching loss. A more general treatment would require extending the theory to weighted Bregman divergences, which cover the diversity of flow-matching objectives in the literature~\cite{lipman2024flow}. In addition, our theoretical study currently applies only to linear couplings; extending the analysis to more sophisticated geometries~\cite{chen2023flow} remains out of scope. Finally, Theorem~\ref{theo:ood_bound} relies on the existence of a continuous limit, whereas practical schedules we consider do not have such a limit.

\paragraph{Noiseless image inverse problems.}
Our work primarily targets noiseless inverse problems, where guidance-induced coupling mismatch is most pronounced and cannot be mitigated by observation noise. In noisy inverse problems, noise implicitly regularizes the posterior and stabilizes Jacobian-based guidance, partially alleviating the issues we address. Moreover, in such settings the effect of coupled guidance interacts with the residual scaling schedule and the noise variance term, making it less straightforward to differentiate improvements due to coupling preservation from those induced by  schedule design. As a result, the relative gains of coupled guidance may be less pronounced in noisy inverse tasks, although the proposed framework remains applicable. Still, since many applications require noiseless inverse guidance (photo editing, content super-resolution, general deblurring), we deem this restriction reasonable.

\section{Conclusion}

We proposed \emph{coupled guidance}, a training-free mechanism that exploits the geometric structure of flow-matching couplings to stabilize guided generation. By jointly leveraging information from both latent and data spaces, our approach achieves competitive or superior performance across robotics and image inverse problems while requiring minimal hyperparameter tuning and negligible computational overhead. 
Extensive experiments show that coupled guidance improves robustness to strongly out-of-distribution guidance signals and guidance sparsity, while improving the quality of generated samples across the board, over all tested dimensionalities and application domains. The observed behaviors validate our theoretical insights and demonstrate that coupled guidance offers a simple, scalable improvement to existing methods. We believe that this approach will extend naturally to more complex coupling structures and that it opens the door for more fine-grained latent control.

\bibliographystyle{plainnat}
\bibliography{main}

\appendix
\onecolumn
\section{Proof of the OOD bound} \label{appendix:Th1_proof}

\begin{theorem}\label{theo:ood_bound_}
    Let  $(\mu_1, \mu_0,\pi,v^\theta)$ be a flow-matching model and let $g$ be a guidance. Under assumptions above, with $A$ the Lipschitz constant of $v \mapsto v^g$,
    \begin{equation}
        \mathbb E_{o}\mathcal W(\mu_1^{g(o)}, \mu_1^{\theta,g(o)}) \leq A \exp\left(\int_{0}^1 L_t dt\right) \mathcal L^{1/2} (\theta)\left\| \frac{d \pi^g}{d\pi}\right\|_{L^\infty}
    \end{equation}
    where $ \pi^g$ is the distribution of coupling $(x_0^g,x_1^g)$ during inference.
\end{theorem}
\begin{proof}
We follow the lines of the proof of Theorem 1 of \citet{benton2023error}. Define $X_{s,t}^x, Y_{s,t}^x$ be the respective flows of  $v$ and $v^\theta$ respectively: 
forall $x\in \mathbb R^d$ and all $0\leq s\leq t\leq 1$
\begin{eqnarray}
    X_{s,t}^x &=& x + \int_{s}^t v^g(X_{s,r}^x,r)dr~; \\
    Y_{s,t}^x &=& x + \int_{s}^t v^{g,\theta}(Y_{s,r}^x,r)dr. \\
\end{eqnarray}
By the Alekseev–Gröbner Theorem we have 
\begin{equation}
    X_{s,t}^x-Y_{s,t}^x = \int_{s}^t\left(\nabla_x  Y_{r,t}^{X_{s,t}^x}\right)(v^g(X_{s,t}^x,r)-v^{g,\theta}(X_{s,t}^x,r))dr.
\end{equation}
Furthermore, 
\begin{equation}
     \nabla_x Y_{s,t}^x= I + \int_{s}^t\left(\nabla_x v^{g,\theta}(Y_{r,t}^x,r)\right)\nabla_x  Y_{r,t}^x dr.
\end{equation}
It follows that by the Gronwall Lemma, 
\begin{equation}
      \|\nabla_x Y_{s,t}^x\|_{\mathrm{op}}\leq \exp\left(\int_{0}^1L_rdr\right):=C.
\end{equation}
Hence, we get:
\begin{eqnarray}
    \|X_{s,t}^x-Y_{s,t}^x\| &\leq& \int_{s}^t\left\|\nabla_x  Y_{r,t}^{X_{s,t}^x}\right\|_{\mathrm{op}}\\&&\quad \times \|v^g(X_{s,t}^x,r)-v^{g,\theta}(X_{s,t}^x,r)\|dr.\nonumber\\ 
     &\leq& C\int_{s}^t \|v^g(X_{s,t}^x,r)-v^{g,\theta}(X_{s,t}^x,r)\|dr.\nonumber
\end{eqnarray}
Taking the expectancy for $x\sim \mu_0 $ and observation $o$ we have 
\begin{eqnarray}
     \mathbb E_{o}\mathbb E_{x}\|X_{0,1}^x-Y_{0,1}^x\| &\leq& C \mathbb E_o\mathbb E_{x} \int_{0}^1 \left\| v^g(X_{0,r}^x,r)-v^{g,\theta}(X_{0,r}^x,r)  \right\|dr\nonumber \\
      \mathbb E_{o}\mathcal W_1(\mu_1^{g(o)}, \mu^{g(o),\theta}_1) &\leq& C \mathbb E_{o}\mathbb E_{t\sim \mathcal U, ( x_0, x_1)\sim (\pi^g|t,o)} \left\| v^g(I( x_0, x_1,t),t)-v^{g,\theta}(I( x_0, x_1,t),t)  \right\| \nonumber\\ &\leq& C \mathbb E_{o}\mathbb E_{t\sim \mathcal U, ( x_0,x_1)\sim (\pi|t,o)} \frac{d(\pi^g|t,o)}{d(\pi|t,o)}\left\| v^g(I( x_0, x_1,t),t)-v^{g,\theta}(I( x_0, x_1,t),t)  \right\| \nonumber
      \\&\leq& C \mathbb E_{o}\mathbb E_{t\sim \mathcal U, ( x_0,x_1)\sim \pi} \frac{d(\pi^g|t,o)}{d\pi}\left\| v^g(I( x_0, x_1,t),t)-v^{g,\theta}(I( x_0, x_1,t),t)  \right\| \nonumber
       \\&\leq& A C \mathbb E_{t\sim \mathcal U, ( x_0,x_1)\sim \pi} \left[\mathbb E_{o}\frac{d(\pi^g|t,o)}{d\pi}\right]\left\| v(I( x_0, x_1,t),t)-v^{\theta}(I( x_0, x_1,t),t)  \right\| \nonumber
        \\&\leq& A C \mathbb E_{t\sim \mathcal U, ( x_0,x_1)\sim \pi} \frac{d(\pi^g|t)}{d\pi}\left\| v(I( x_0, x_1,t),t)-v^{\theta}(I( x_0, x_1,t),t)  \right\| \nonumber
        \\&\leq& A C \sqrt{\mathbb E_{t\sim \mathcal U, ( x_0,x_1)\sim \pi} \left(\frac{d(\pi^g|t)}{d\pi}\right)^2 \mathbb E_{t\sim \mathcal U, ( x_0,x_1)\sim \pi}\left\| v(I( x_0, x_1,t),t)-v^{\theta}(I( x_0, x_1,t),t)  \right\|^2} \nonumber
          \\&\leq& A C \left\|\frac{d\pi^g}{d\pi}\right\|_{\infty} \sqrt{\mathcal L(\theta)\nonumber}
\end{eqnarray}
\end{proof}

\section{Experimental details}

\subsection{Baseline selection}\label{appendix:baselines}

As stated in the main text, the baselines for training-free guidance are scarce. We took \citet{feng2025guidance}'s five variants as our starting point and considered each carefully. 

First $g^{MC}$ requires a large number of parallel samples per step, making it computationally intractable at our experiment's dimensionalities, an issue which the original paper also faces. For most of our experiments (super-resolution and inpainting and our guidance sparsity experiments), $g^{cov-A}$ and $g^{cov-G}$ reduce analytically to $\Pi$GDM and $g^{\text{sim-inv-A}}$ (which we refer to as $\Pi$GDM+) respectively. We hence choose to benchmark our method on these two baselines as they cover all 4 tractable methods from \citet{feng2025guidance}.

This setup also covers FlowChef  \cite{patel2024steeringrectifiedflowmodels}, which gives proportional guidance to $g^{cov-A}$, and in turn $\Pi$GDM, with a factor that is absorbed during our tuning. We choose not to include PnP-Flow \cite{martin2025pnpflowplugandplayimagerestoration}, which is broadly outperformed by FlowChef, and in turn by \citet{feng2025guidance}. The remaining training-free guidance works \cite{erbach2025solvinginverseproblemsflair,park2025flowlpslangevinproximalsamplingflowbased,kalaivanan2025essflowtrainingfreeguidanceflowbased} operate exclusively in latent space, rely on noise-level terms incompatible with noiseless tasks, or do not scale to high dimension tasks; none of them are directly applicable here. 
 
For the baseline choice of collision avoidance, we cannot directly use \pigdm, as it would require finding an \textit{adjoint operator}. Concretely, this means that we would need to find a function $h$ such that, with $p$ the collision avoidance pipeline, $p \circ h \circ p(x) = p(x)$. This means that given the environment and the value of the Euclidean Signed Distance Function, we need to be able to find a set of joint positions that solves to the same ESDF. This is highly non trivial, and if implemented would have a prohibitive cost to use in a robotics setting. Since basic energy minimization does not have such a constraint on the guidance operator, we use it in its simplest form :
$$g_{\text{collision}}=-\lambda_t\nabla_{\widehat x_1}||\widehat x_1-y||$$

\subsection{Robotics}\label{appendix:robot_experiments}
For the robotics tasks we use the same UNet architecture as in \citet{chi2025diffusion}, with 3 levels of depth respectively 256, 512 and 1024. They are trained with a time resolution of $1/128$ for 1000 epochs on the "expert-v2" datasets available in the Minari python library \cite{minari}.  The reported scores use 300 environment seeds.

\subsection{Image Generation} \label{appendix:image_experiments}
For the image generation tasks, a 2 billion parameter UNet is used as defined in the code provided by \citet{feng2025guidance}, with minor fixes to the training code to make it more stable. They are trained using a batch of size 128, base learning rate $10^{-4}$ with AdamW \cite{loshchilov2017decoupled}, warm-up then cosine annealing \cite{loshchilov2016sgdr}.
  For both training and inference we tested and found negligible performance impacts for running the model in mixed precision: the training is done using torch automatic mixed precision, and all methods use BF16 \cite{tagliavini2018transprecision} for the reported scores. We noticed an improvement in base model performance when normalizing data which was not done in \citet{feng2025guidance}'s code. The models are hence trained with normalized data, but we noticed that both \pigdm and \gsim suffered from doing guidance in normalized space with guidance scale varying across two orders of magnitude from task to task. To fairly evaluate the baselines, the model is called in normalized space but the velocity field is then unnormalized before applying guidance:
\begin{equation}
    v_t = \sigma v_\theta\left(\frac{x_{raw}-\mu}{\sigma},t\right)
\end{equation}
where $\mu$ and $\sigma$ are the dataset's mean and standard deviation. Finally, all image metrics are calculated using torchmetrics \cite{Detlefsen2022TorchMetricsM}.

\newpage
\section{Additional results}\label{appendix:results}
\subsection{Robotics}
\label{appendix:robot_results}
\FloatBarrier 
In the robotics setting, our method improves robustness to sparsity for \pigdm and \pigdmp in a similar way, for all 4 tasks.
\begin{figure}[h!]
    \centering
    \includegraphics[width=0.7\linewidth]{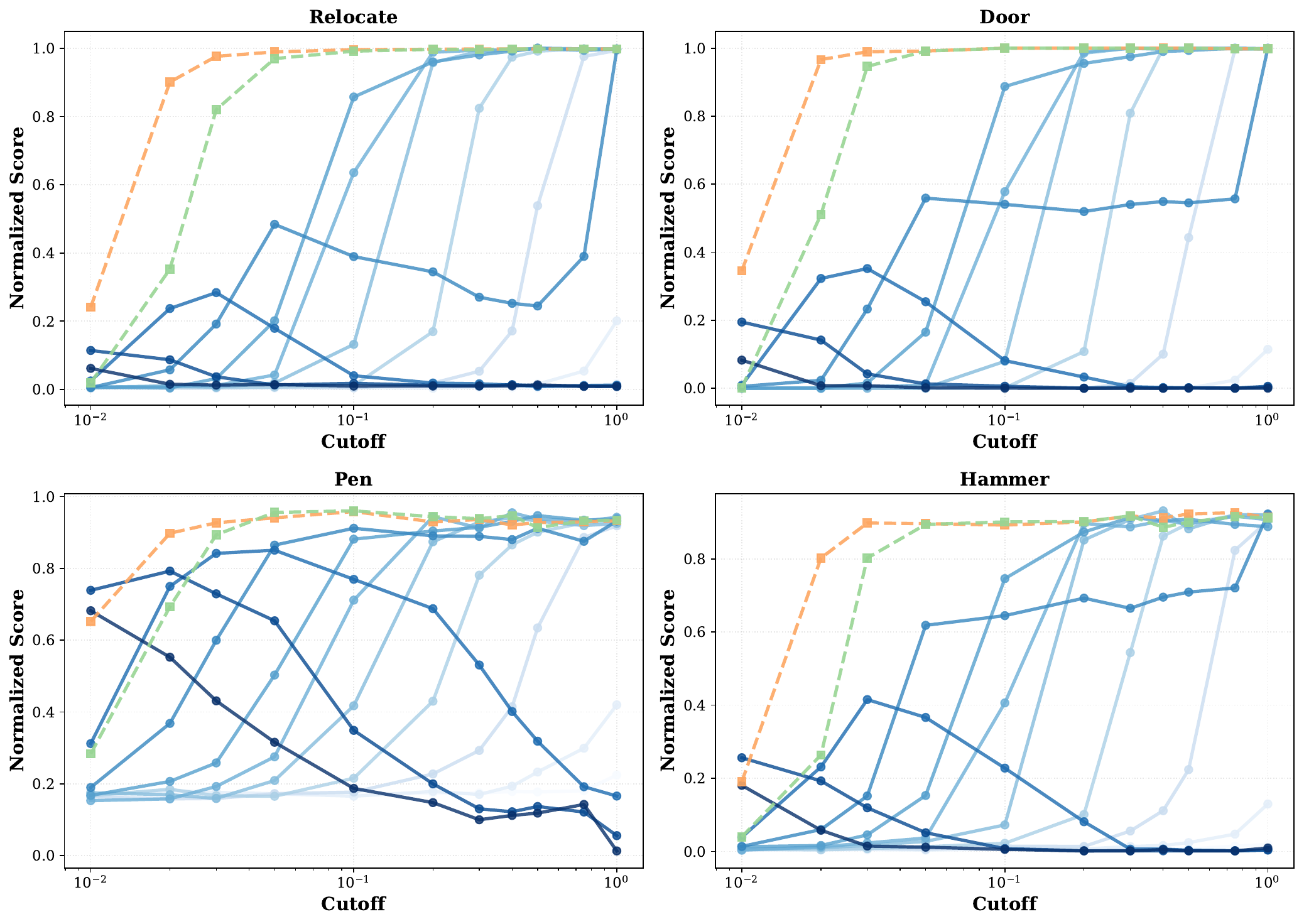}
    \includegraphics[width=0.7\linewidth]{figures/out_pigdm_legend.pdf}
    \caption{Sparsity tests for all Adroit tasks using \pigdmp: the cutoff represents the proportion of guided coordinates. Different shades of blue represent different tunings of the guidance scale for the naive guidance runs ($\beta=1$) }
    \label{fig:placeholder_1}
\end{figure}
\FloatBarrier

\begin{figure}[h!]
    \centering
    \includegraphics[width=0.7\linewidth]{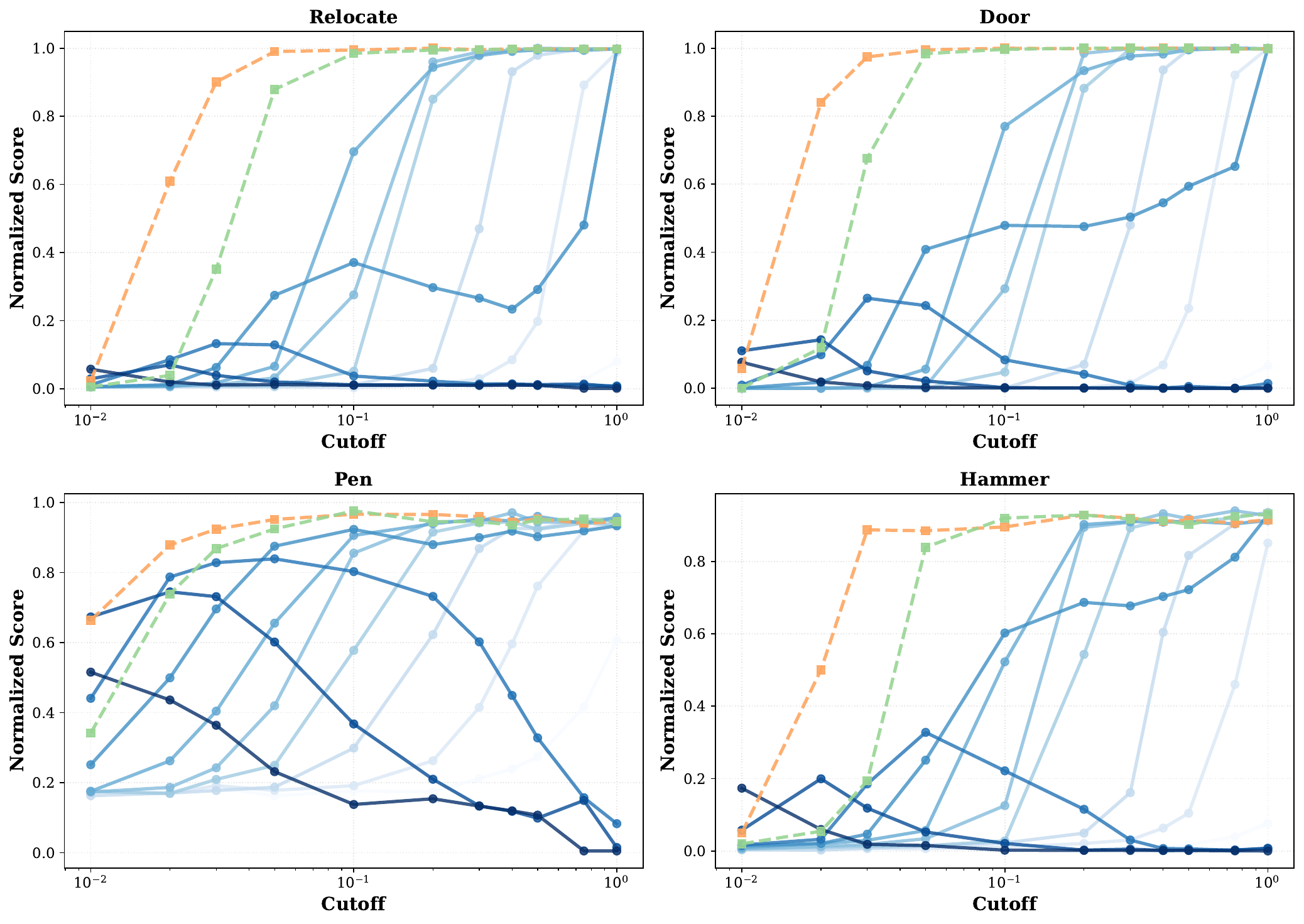}
    \includegraphics[width=0.7\linewidth]{figures/out_pigdm_legend.pdf}
    \caption{Sparsity tests for all Adroit tasks using \pigdm: the cutoffs represent the proportion of guided coordinates. Different shades of blue represent different tunings of the guidance scale for the naive guidance runs ($\beta=1$) }
    \label{fig:placeholder_2}
\end{figure}

\clearpage
\subsection{Images}
\label{appendix:images}
All image to image scores (LPIPS, PSNR, SSIM) behave in similar ways for different tunings and $\beta$ values. The plots look slightly different for the FID scores which can be expected: very strong guidances lead to divergence and artifacts, so the strong guidance cliff is the same. However in the low guidance regions, the model is basically not guided, leading to high quality samples, hence giving a decent FID score. 

\begin{table}[h]
\centering
\caption{Comparison of naive vs.\ coupled guidance across methods and tasks on CelebA-HQ-256.}
\label{tab:results}
\resizebox{0.6\columnwidth}{!}{%
\begin{tabular}{llcccc}
\toprule
 & & \multicolumn{2}{c}{\textbf{Naive}} & \multicolumn{2}{c}{\textbf{Coupled}} \\
\cmidrule(lr){3-4} \cmidrule(lr){5-6}
\textbf{Task} & \textbf{Metric} & PiGDM & PiGDM+ & PiGDM & PiGDM+ \\
\midrule
\multirow{4}{*}{Inpainting}
  & FID$\downarrow$   & 54.3  & 7.5   & 20.7           & \textbf{7.5} \\
  & LPIPS$\downarrow$ & 0.184 & \textbf{0.041} & 0.100          & \textbf{0.041} \\
  & PSNR$\uparrow$    & 20.5  & 29.3  & \textbf{30.1}  & 29.1 \\
  & SSIM$\uparrow$    & 0.873 & \textbf{0.953} & 0.927 & \textbf{0.953} \\
\midrule
\multirow{4}{*}{Deblurring}
  & FID$\downarrow$   & 54.0  & \textbf{22.7 } & 33.0  & \textbf{22.7 }\\
  & LPIPS$\downarrow$ & 0.462 & 0.153 & 0.209 & \textbf{0.148} \\
  & PSNR$\uparrow$    & 20.1  & 30.0  & 28.6  & \textbf{30.6} \\
  & SSIM$\uparrow$    & 0.472 & 0.830 & 0.784 & \textbf{0.840} \\
\midrule
\multirow{4}{*}{Super-res.}
  & FID$\downarrow$   & 55.9  & 76.9  & \textbf{25.6}  & 77.0 \\
  & LPIPS$\downarrow$ & 0.502 & 0.486 & \textbf{0.166} & 0.431 \\
  & PSNR$\uparrow$    & 18.1  & 19.6  & \textbf{28.8}  & 19.3 \\
  & SSIM$\uparrow$    & 0.449 & 0.443 & \textbf{0.833} & 0.459 \\
\bottomrule
\end{tabular}%
}
\end{table}

\begin{table}[h]
\centering
\caption{Comparison of naive vs.\ coupled guidance across methods and tasks on ImageNet (CFM).}
\label{tab:results_imagenet_cfm}
\resizebox{0.6\columnwidth}{!}{%
\begin{tabular}{llcccc}
\toprule
 & & \multicolumn{2}{c}{\textbf{Naive}} & \multicolumn{2}{c}{\textbf{Coupled}} \\
\cmidrule(lr){3-4} \cmidrule(lr){5-6}
\textbf{Task} & \textbf{Metric} & $\Pi$GDM & $\Pi$GDM+ & $\Pi$GDM & $\Pi$GDM+ \\
\midrule
\multirow{4}{*}{Inpainting}
  & FID$\downarrow$   & 13.0  & 10.5 & \textbf{8.8}  & 11.3 \\
  & LPIPS$\downarrow$ & 0.106 & \textbf{0.086} & 0.092 & 0.088 \\
  & PSNR$\uparrow$    & 20.4  & 20.5  & \textbf{23.4}  & 20.6 \\
  & SSIM$\uparrow$    & 0.907 & \textbf{0.911} & \textbf{0.908} & \textbf{0.909} \\
\midrule
\multirow{4}{*}{Deblurring}
  & FID$\downarrow$   & 58.6  & 1.6   & 16.6  & \textbf{1.2} \\
  & LPIPS$\downarrow$ & 0.452 & 0.029 & 0.232 & \textbf{0.021} \\
  & PSNR$\uparrow$    & 19.4  & 32.9  & 22.9  & \textbf{34.0} \\
  & SSIM$\uparrow$    & 0.474 & 0.969 & 0.749 & \textbf{0.976} \\
\midrule
\multirow{4}{*}{Super-res.}
  & FID$\downarrow$   & 20.7  & 49.3 & \textbf{8.9}  & 50.6 \\
  & LPIPS$\downarrow$ & 0.230 & 0.411 & \textbf{0.150} & 0.418 \\
  & PSNR$\uparrow$    & 21.8  & 17.2  & \textbf{23.9}  & 17.2 \\
  & SSIM$\uparrow$    & 0.741 & 0.458 & \textbf{0.811} & 0.445 \\
\bottomrule
\end{tabular}%
}
\end{table}

\begin{figure*}[h!]
    \centering
    \includegraphics[width=0.31\linewidth]{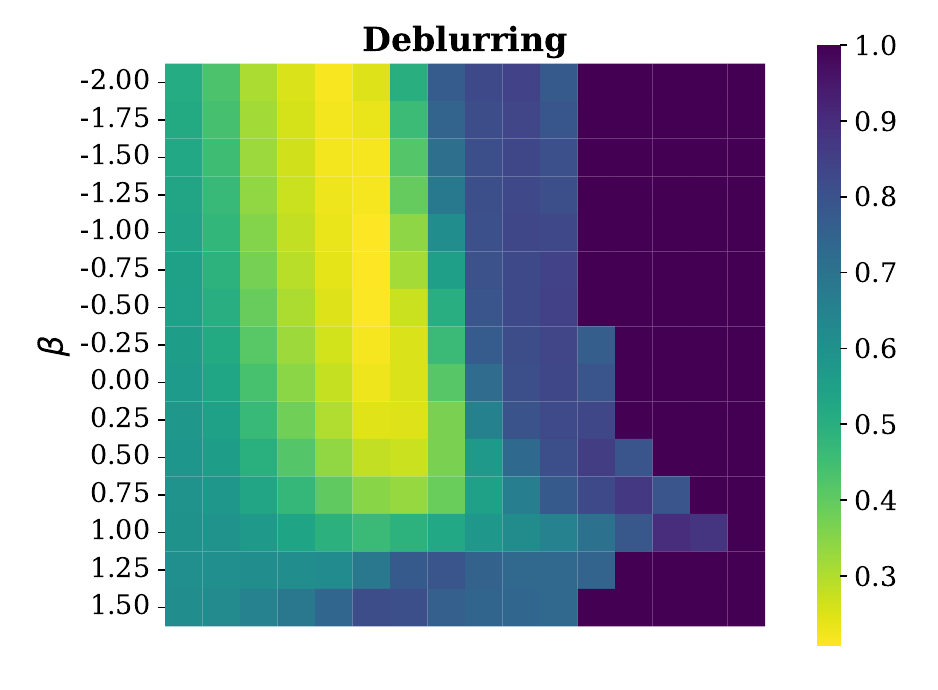}
    \includegraphics[width=0.31\linewidth]{figures/plots_PiGDM_superresolution_lpips_score_heatmap.pdf}
    \includegraphics[width=0.31\linewidth]{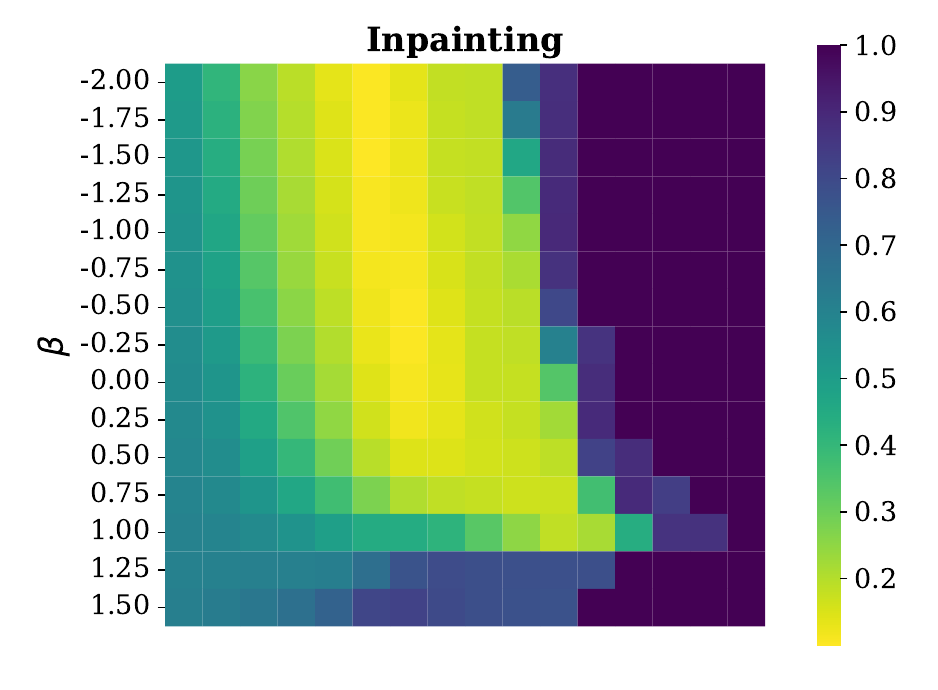}
    \includegraphics[width=0.31\linewidth]{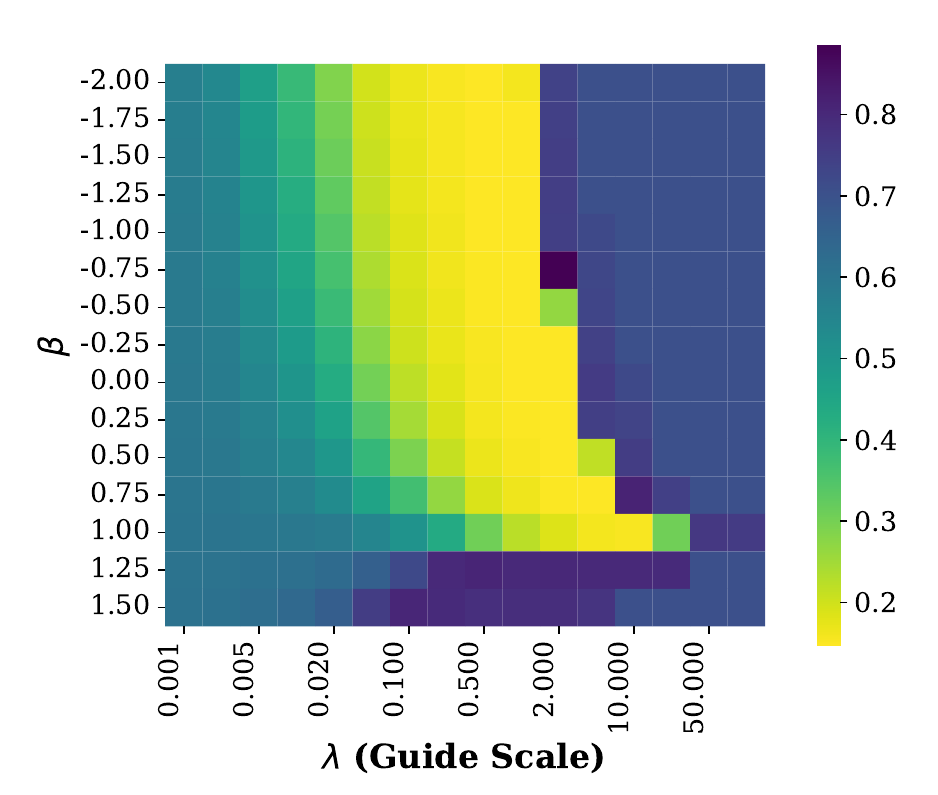}
    \includegraphics[width=0.31\linewidth]{figures/plots_PiGDM+_superresolution_lpips_score_heatmap.pdf}
    \includegraphics[width=0.31\linewidth]{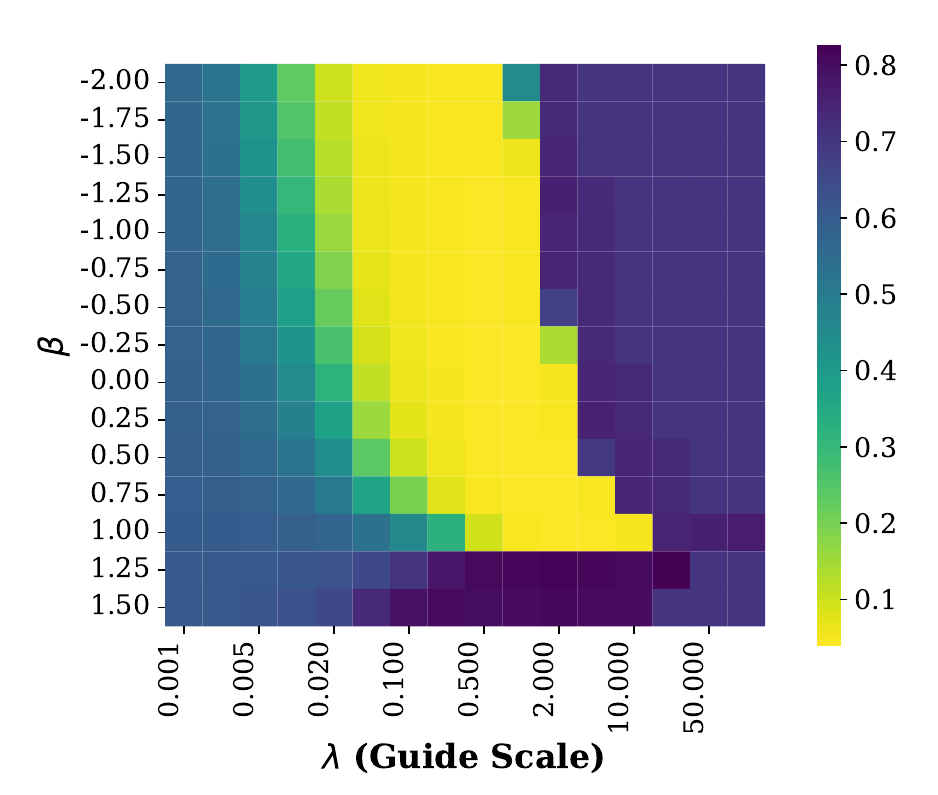}
    \caption{Sample fidelity using different tuning scales and $\beta$ values for deblurring, superresolution and inpainting (left to right). First row shows the LPIPS scores for \pigdm, second row shows the performance of \pigdmp}
    \label{fig:heatmaps_appendix_1}
\end{figure*}
\begin{figure*}[h!]
    \centering
    \includegraphics[width=0.31\linewidth]{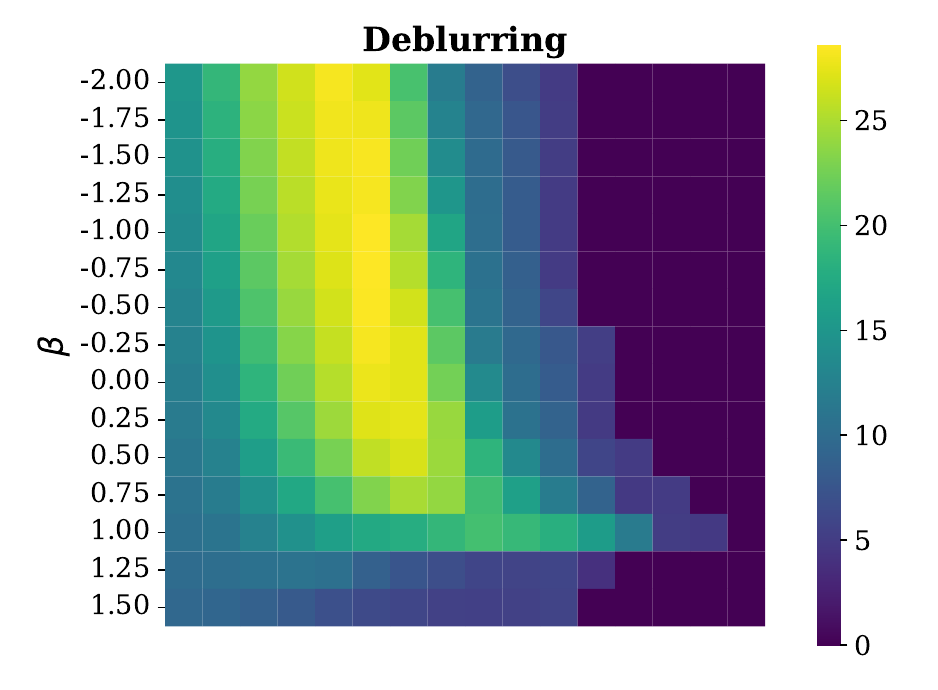}
    \includegraphics[width=0.31\linewidth]{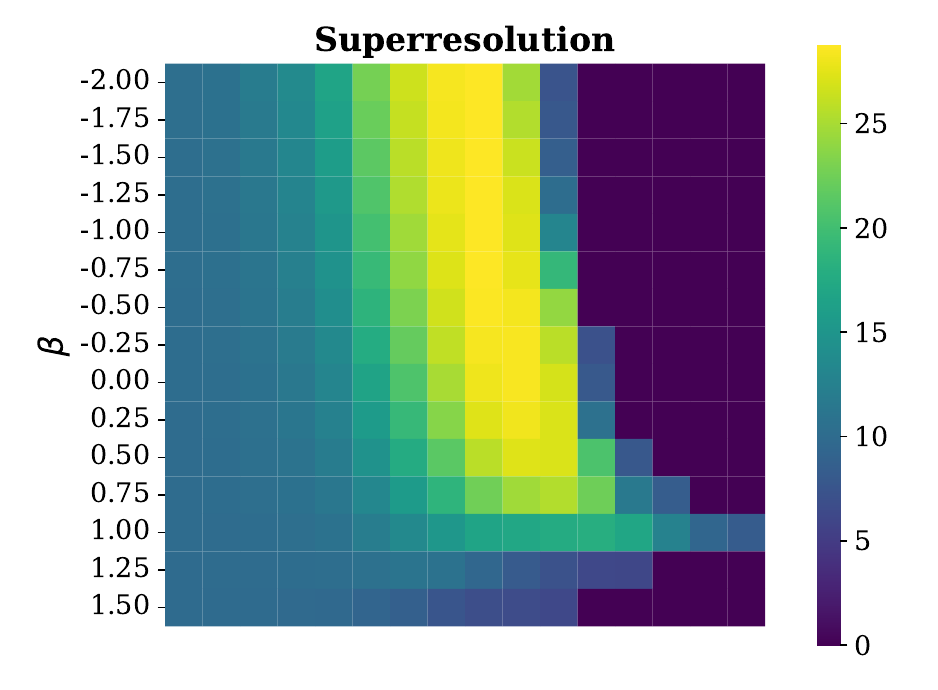}
    \includegraphics[width=0.31\linewidth]{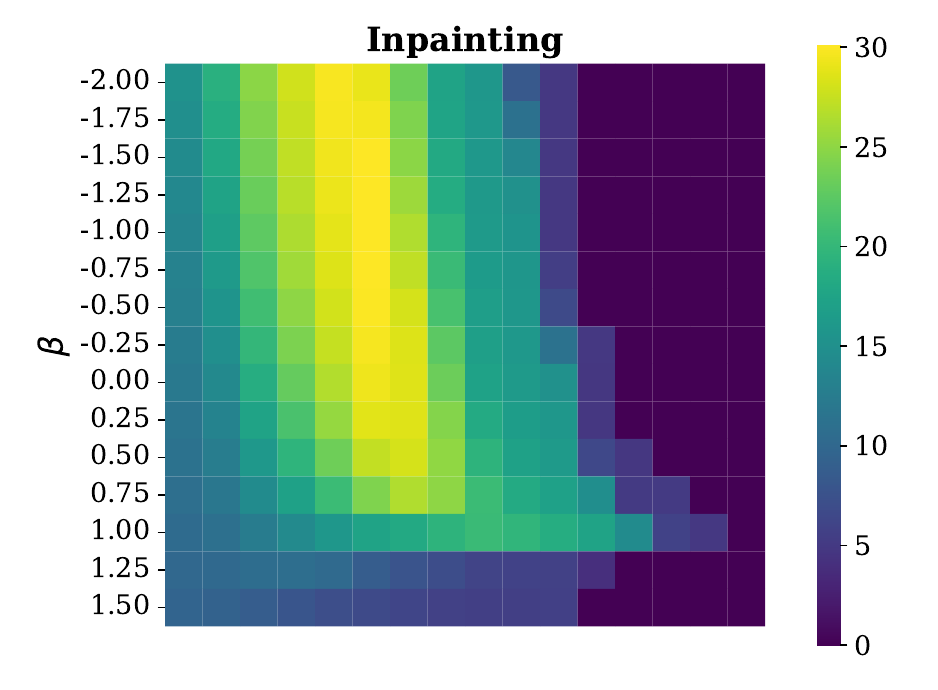}
    \includegraphics[width=0.31\linewidth]{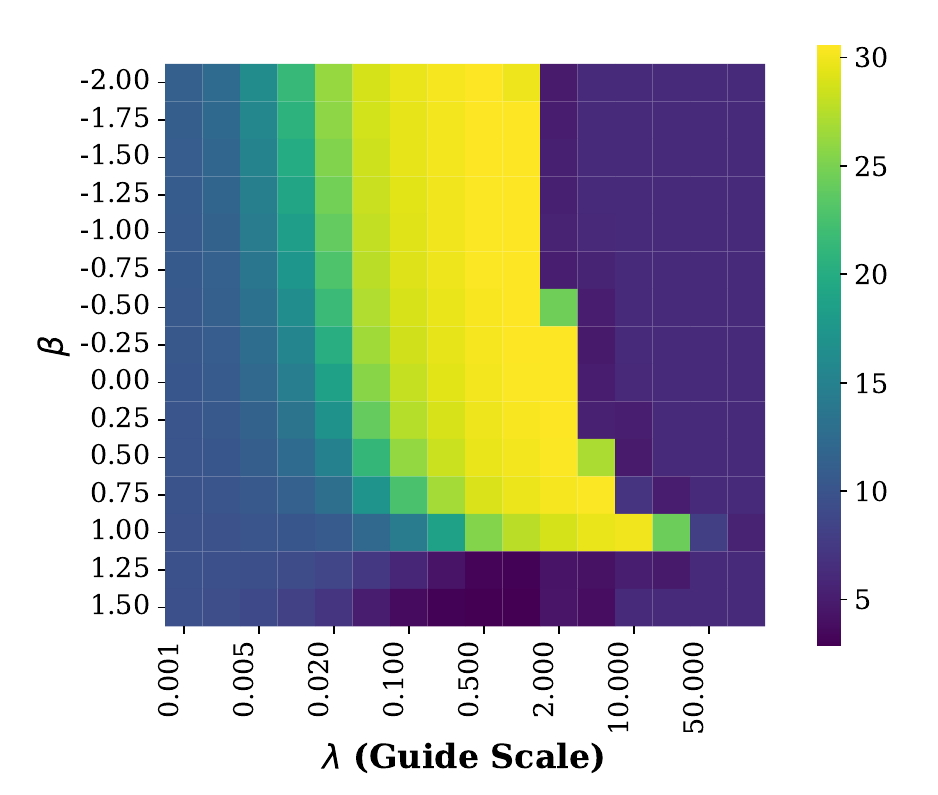}
    \includegraphics[width=0.31\linewidth]{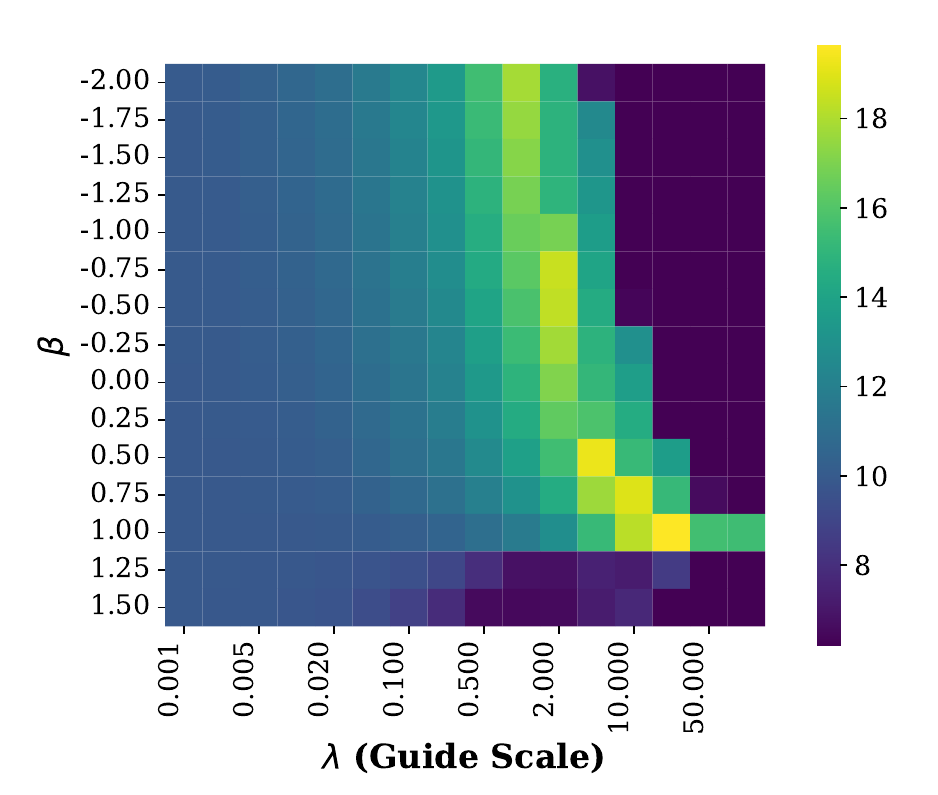}
    \includegraphics[width=0.31\linewidth]{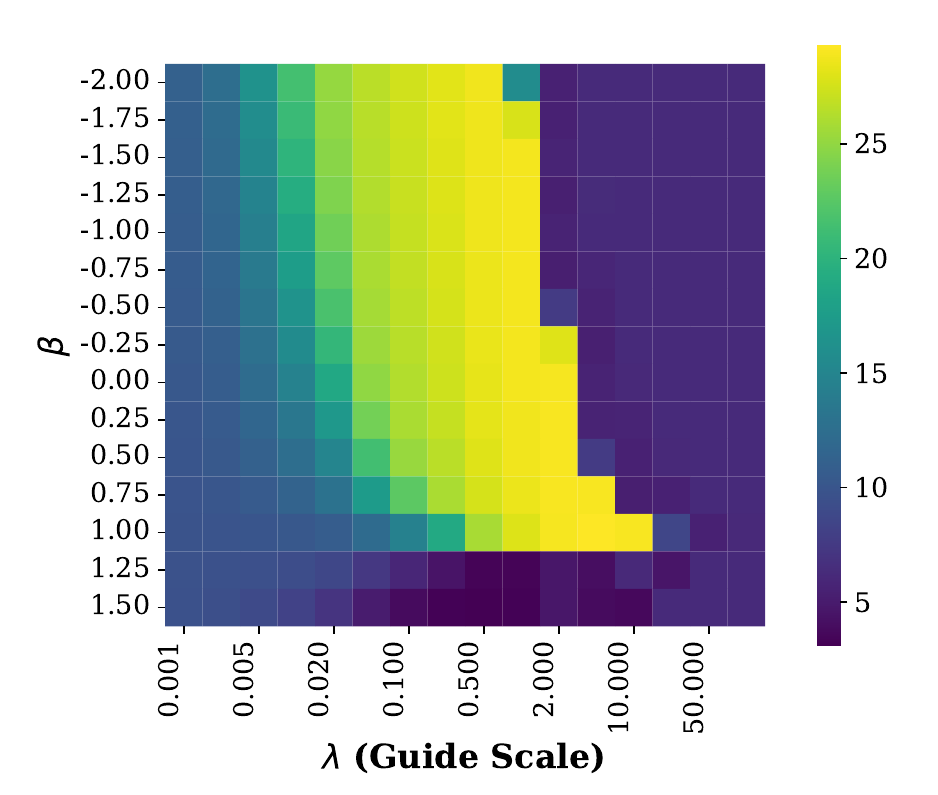}
    \caption{Sample fidelity using different tuning scales and $\beta$ values for deblurring, superresolution and inpainting (left to right). First row shows the PSNR scores for \pigdm, second row shows the performance of \pigdmp}
    \label{fig:heatmaps_appendix_2}
\end{figure*}
\begin{figure*}[h!]
    \centering
    \includegraphics[width=0.31\linewidth]{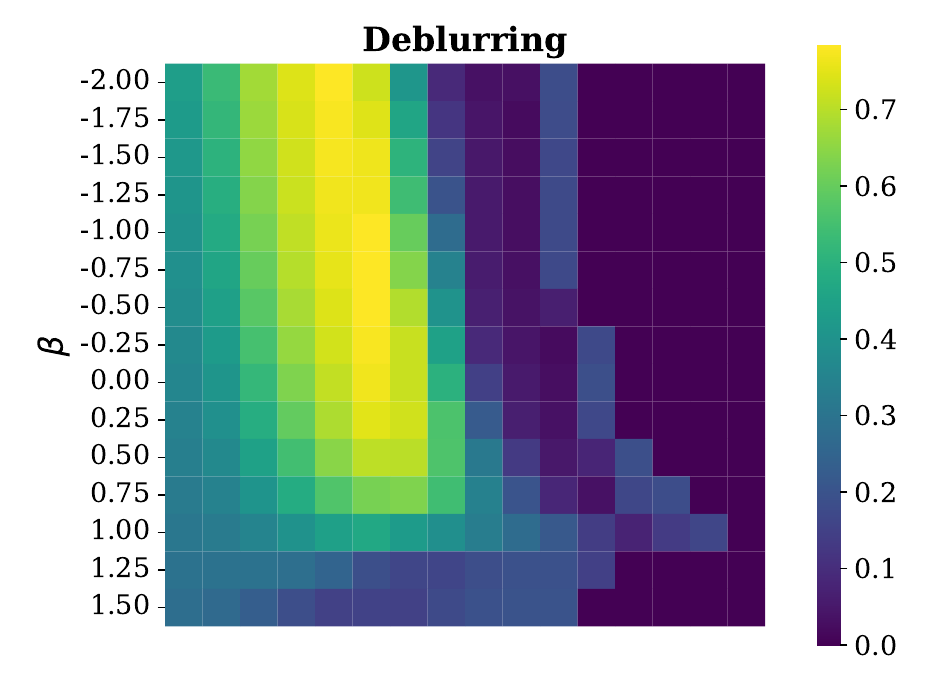}
    \includegraphics[width=0.31\linewidth]{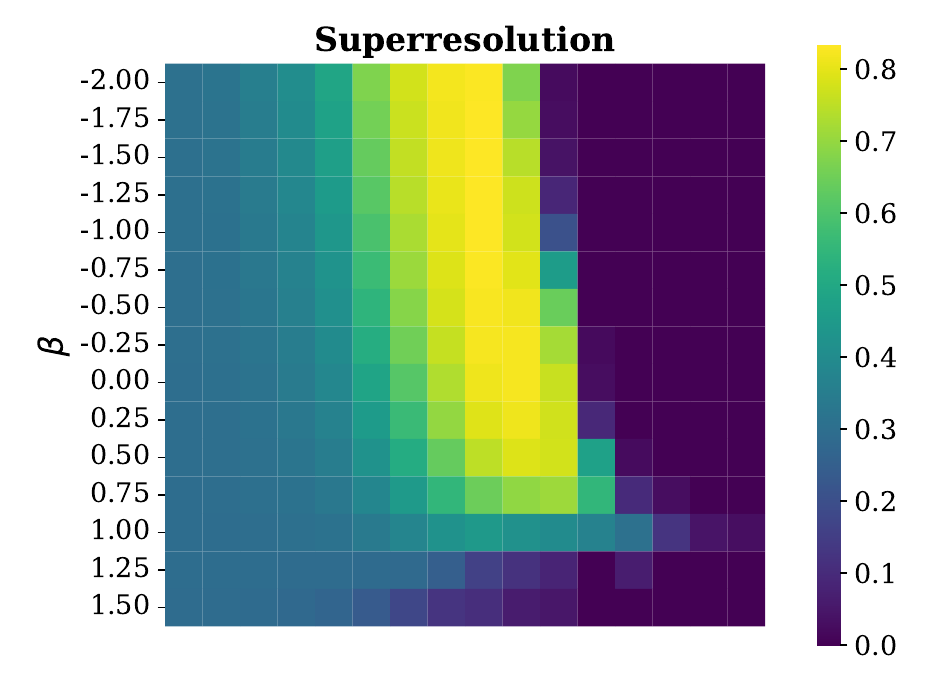}
    \includegraphics[width=0.31\linewidth]{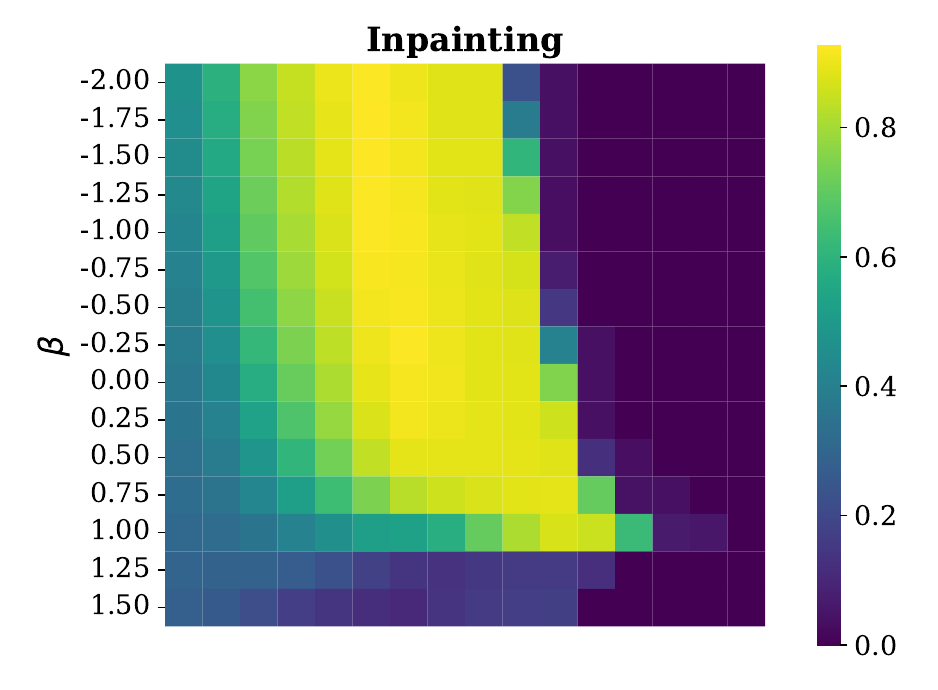}
    \includegraphics[width=0.31\linewidth]{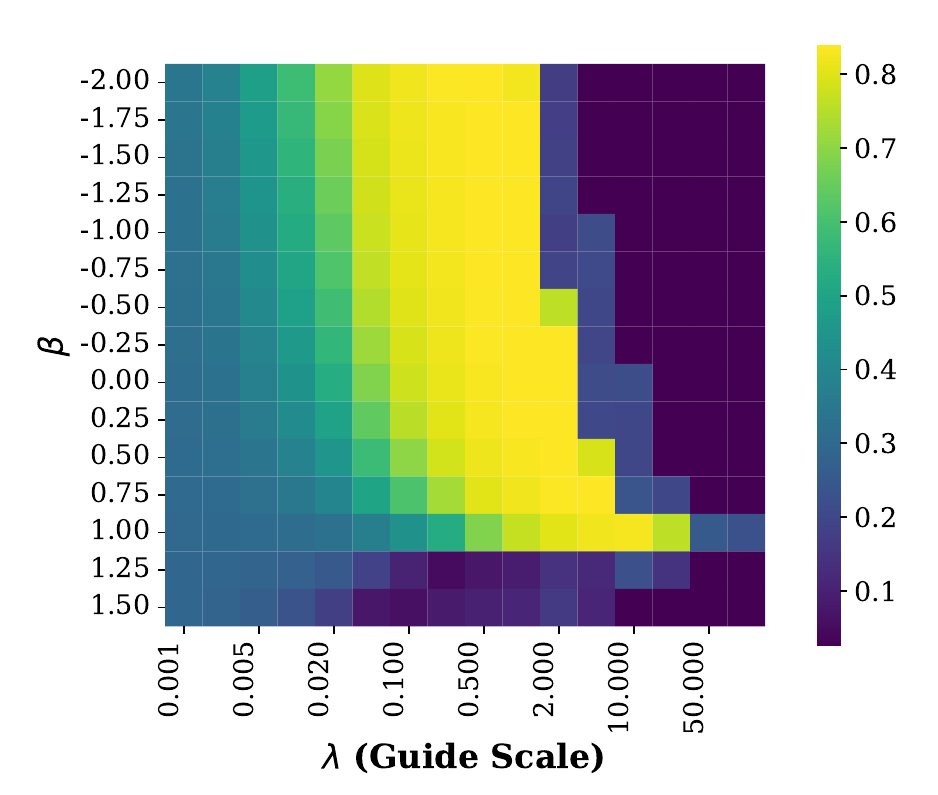}
    \includegraphics[width=0.31\linewidth]{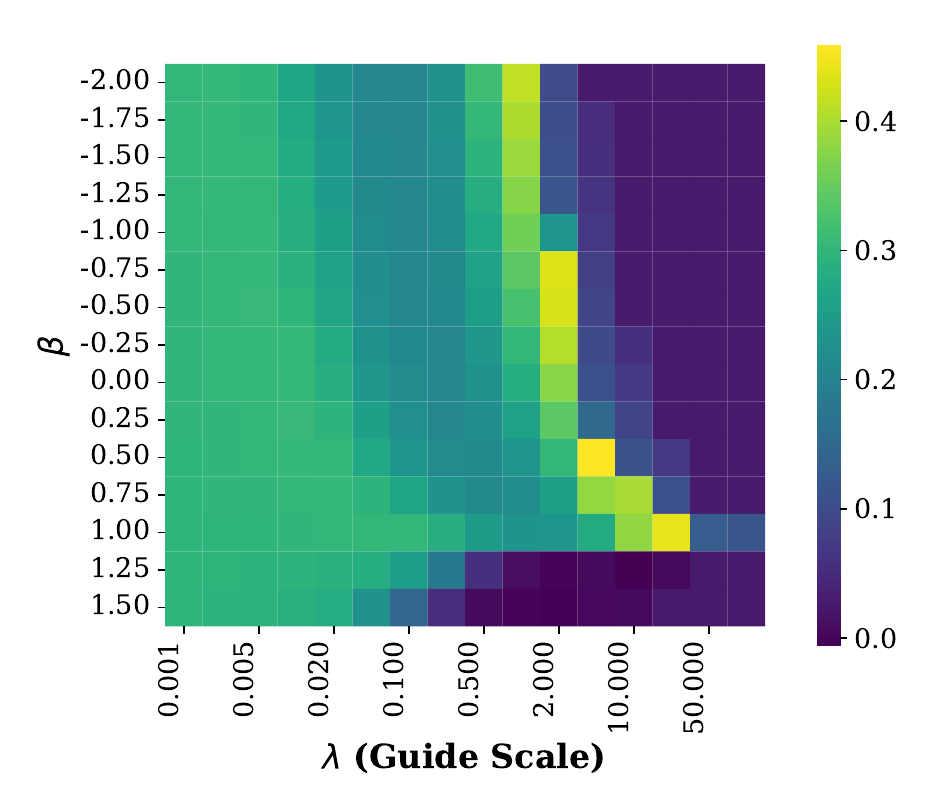}
    \includegraphics[width=0.31\linewidth]{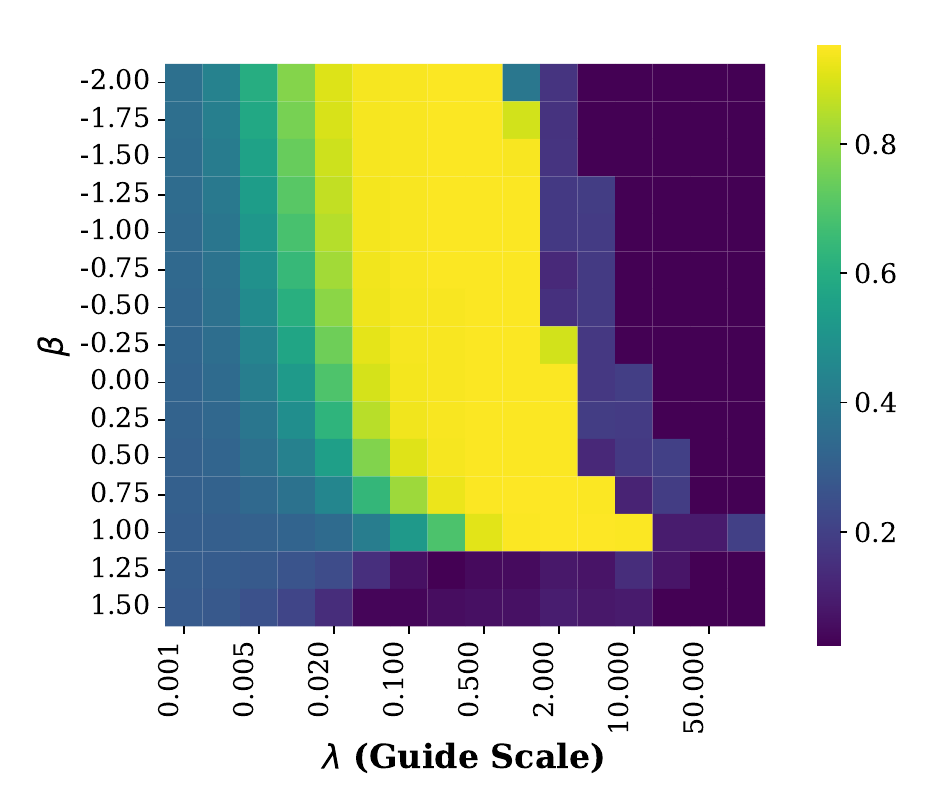}
    \caption{Sample fidelity using different tuning scales and $\beta$ values for deblurring, superresolution and inpainting (left to right). First row shows the SSIM scores for \pigdm, second row shows the performance of \pigdmp}
    \label{fig:heatmaps_appendix_3}
\end{figure*}
\begin{figure*}[h!]
    \centering
    \includegraphics[width=0.31\linewidth]{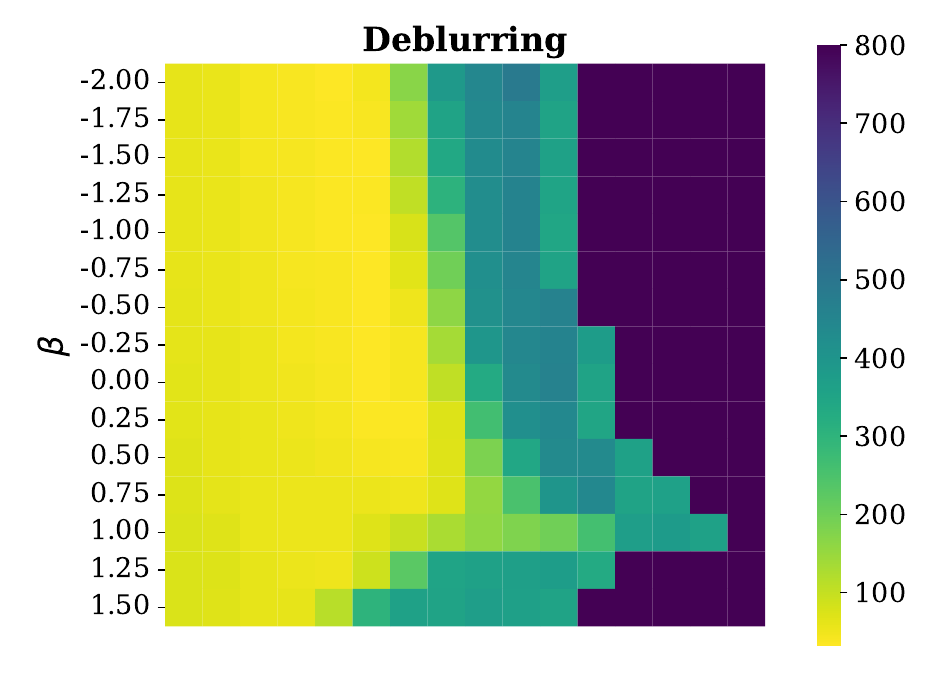}
    \includegraphics[width=0.31\linewidth]{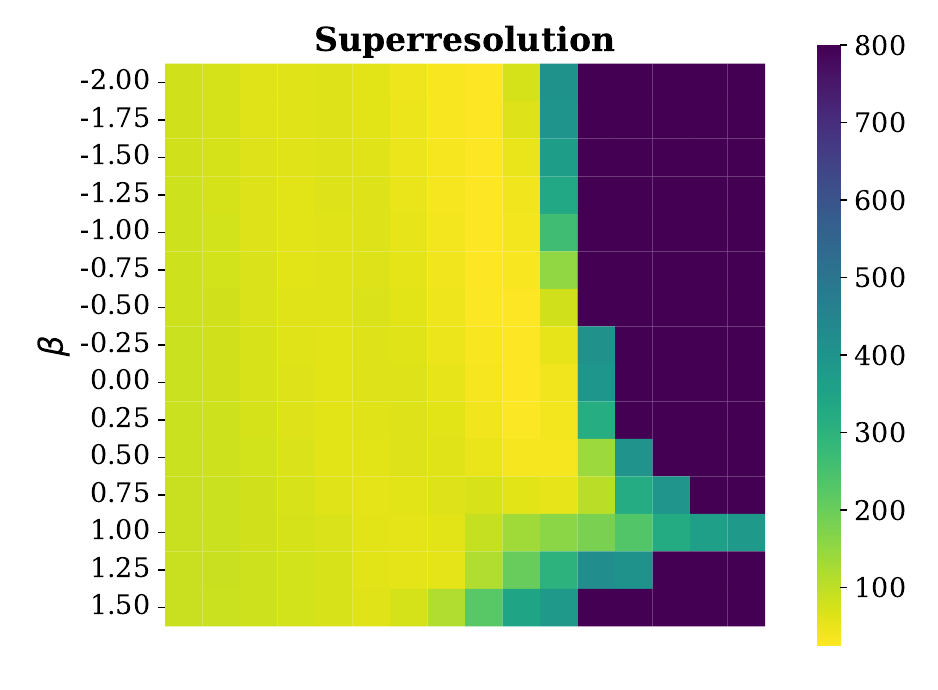}
    \includegraphics[width=0.31\linewidth]{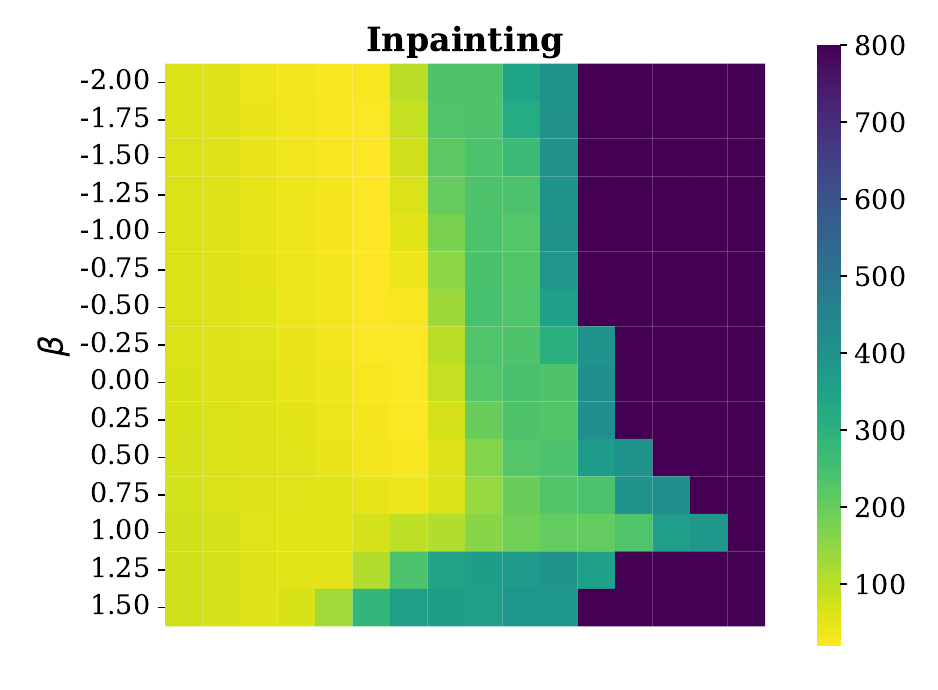}
    \includegraphics[width=0.31\linewidth]{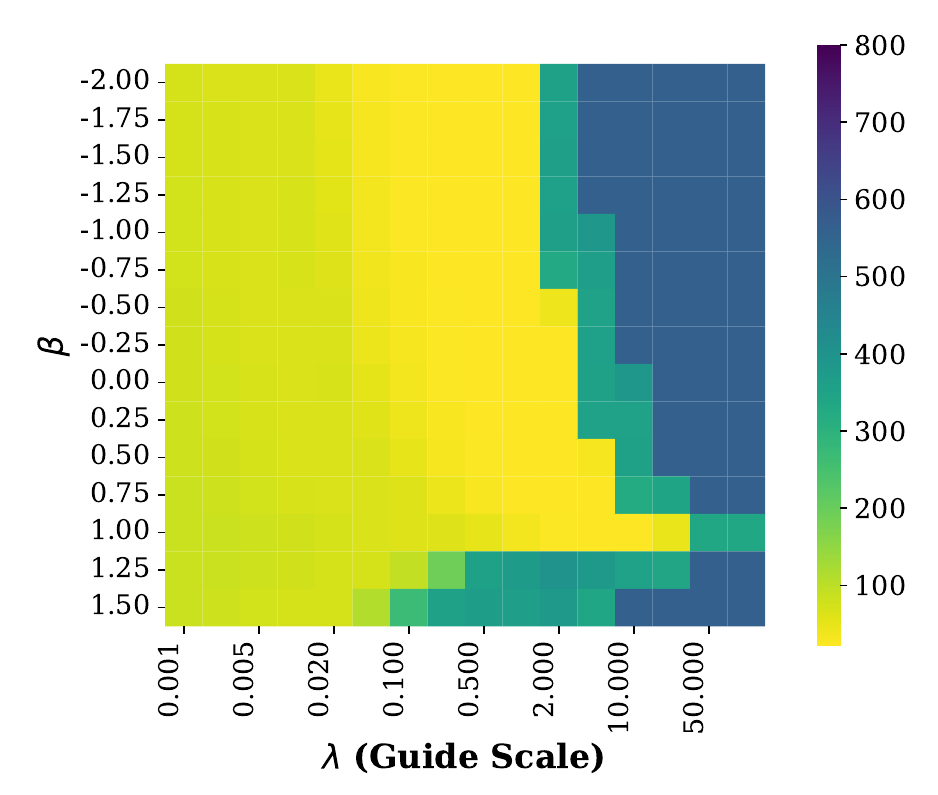}
    \includegraphics[width=0.31\linewidth]{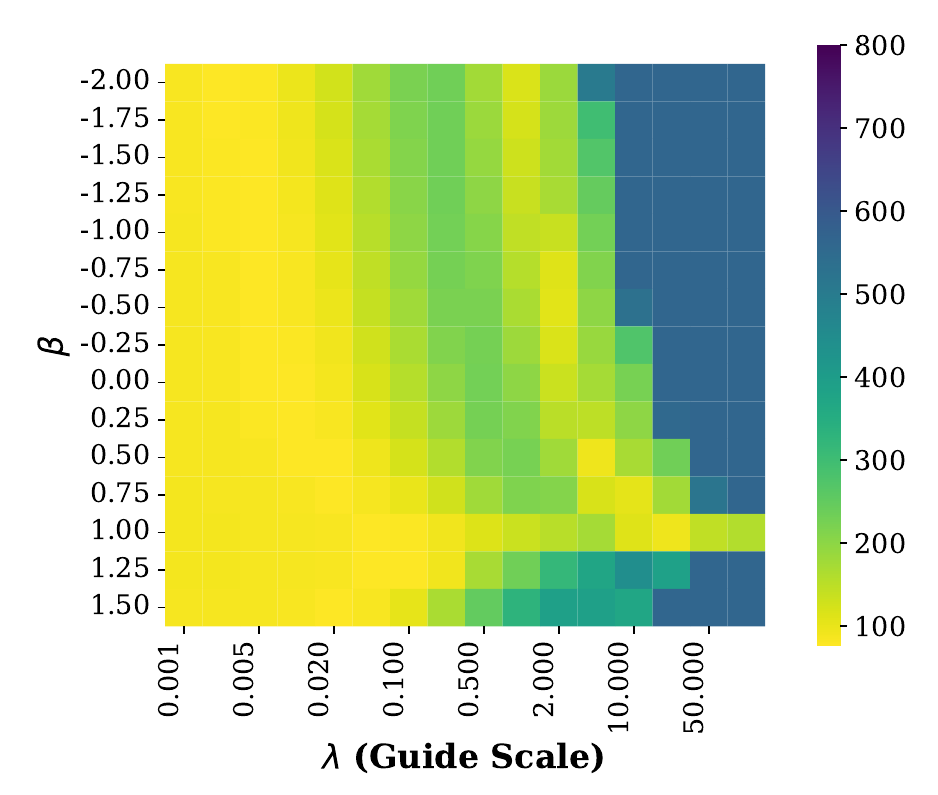}
    \includegraphics[width=0.31\linewidth]{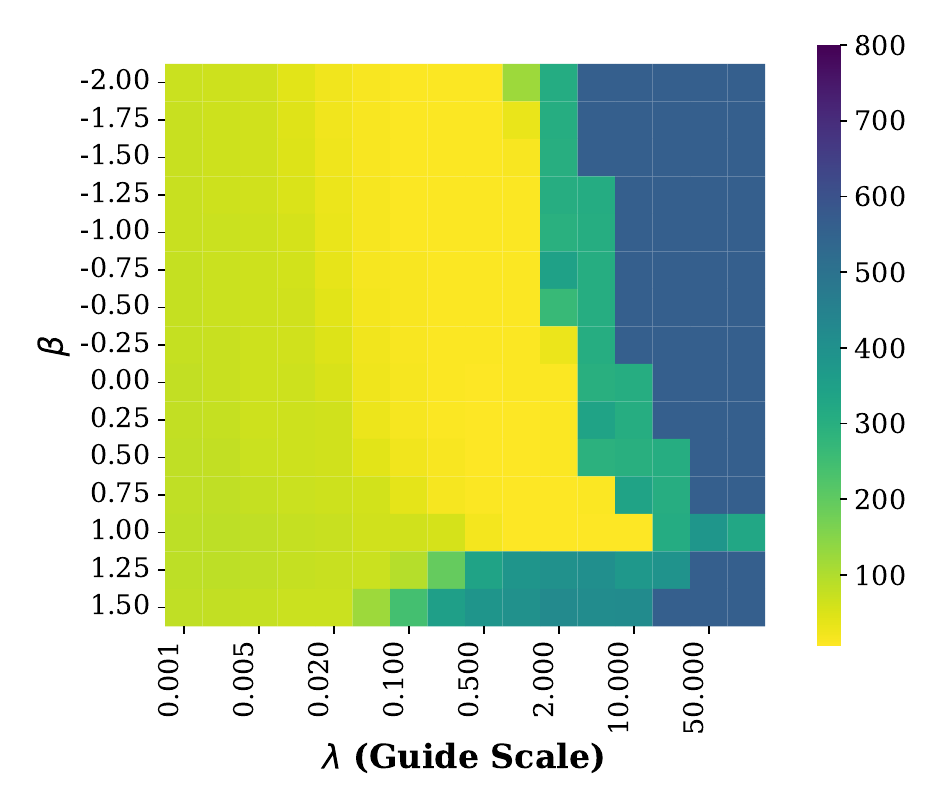}
    \caption{Sample fidelity using different tuning scales and $\beta$ values for deblurring, superresolution and inpainting (left to right). First row shows the FID scores for \pigdm, second row shows the performance of \pigdmp}
    \label{fig:heatmaps_appendix_4}
\end{figure*}
\FloatBarrier
We also provide visualizations for both $\widehat x_1$ through inference and final samples for all tasks and a range of beta values. These samples show the difference in convergence between naive and coupled guidance and help understand where the artifacts created by naive guidance come from.

\begin{figure}
    \centering
    \includegraphics[width=0.72\linewidth]{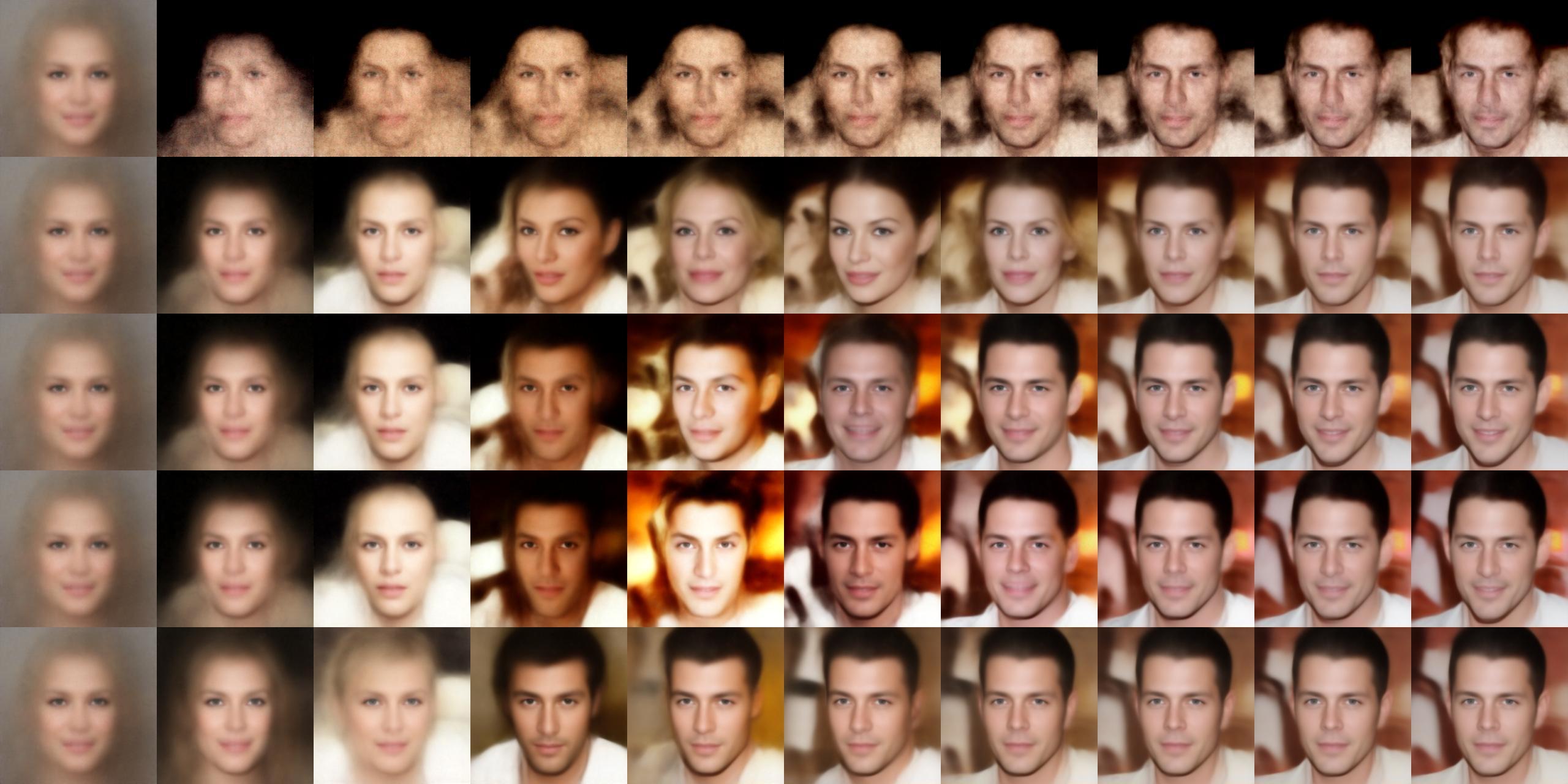}
    \caption{$\widehat x_1$ predictions using \pigdm for the first 10 steps of deblurring showing the convergence of the guided process for $\beta \in [1,0.75,0.5,0.25,0]$ (from top to bottom)}
    \label{fig:placeholder_3}
\end{figure}

\begin{figure}
    \centering
    \includegraphics[width=0.72\linewidth]{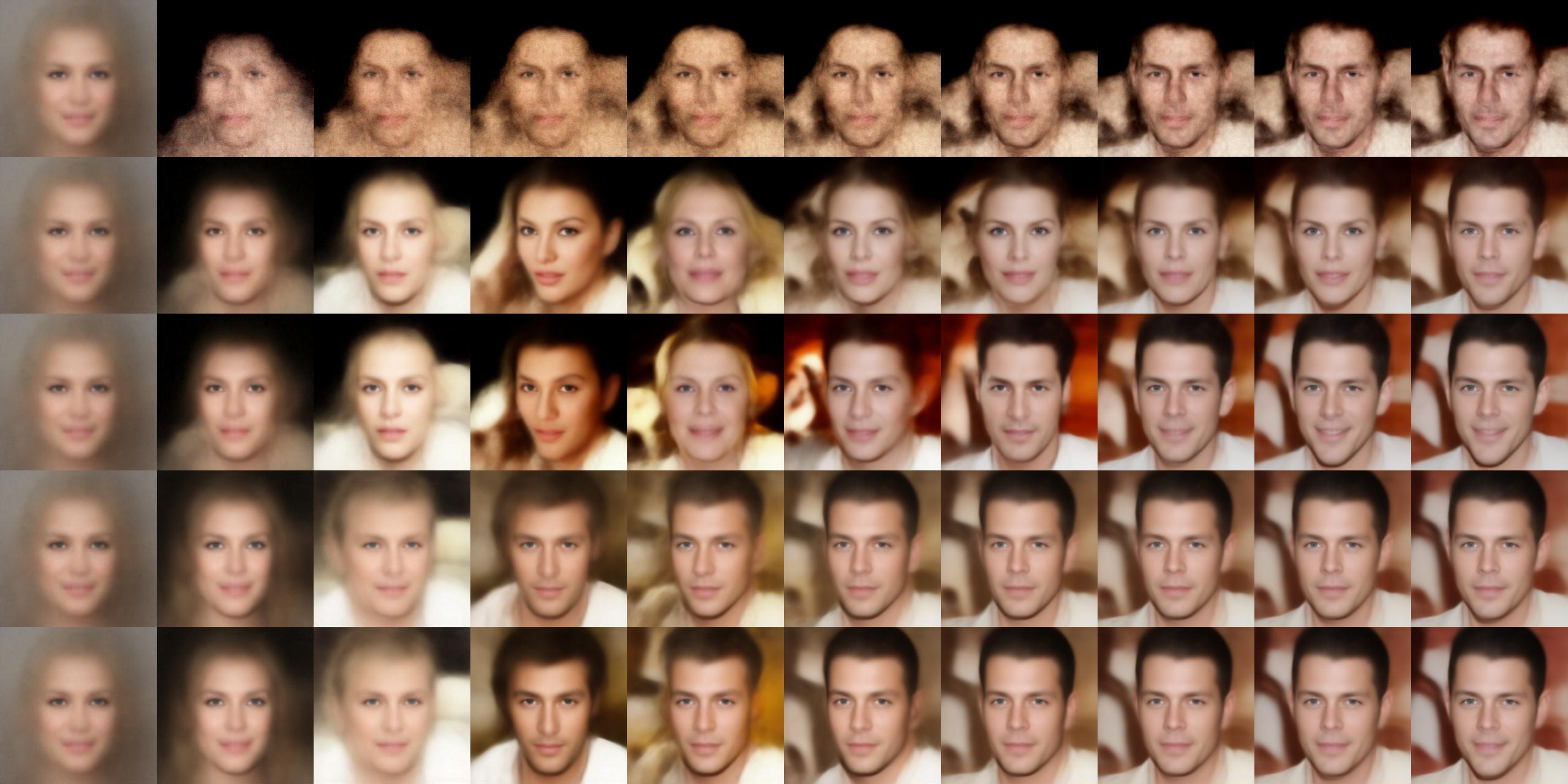}
    \caption{$\widehat x_1$ predictions using \pigdm for the first 10 steps of inpainting showing the convergence of the guided process for $\beta \in [1,0.75,0.5,0.25,0]$ (from top to bottom)}
    \label{fig:placeholder_4}
\end{figure}

\begin{figure}
    \centering
    \includegraphics[width=0.72\linewidth]{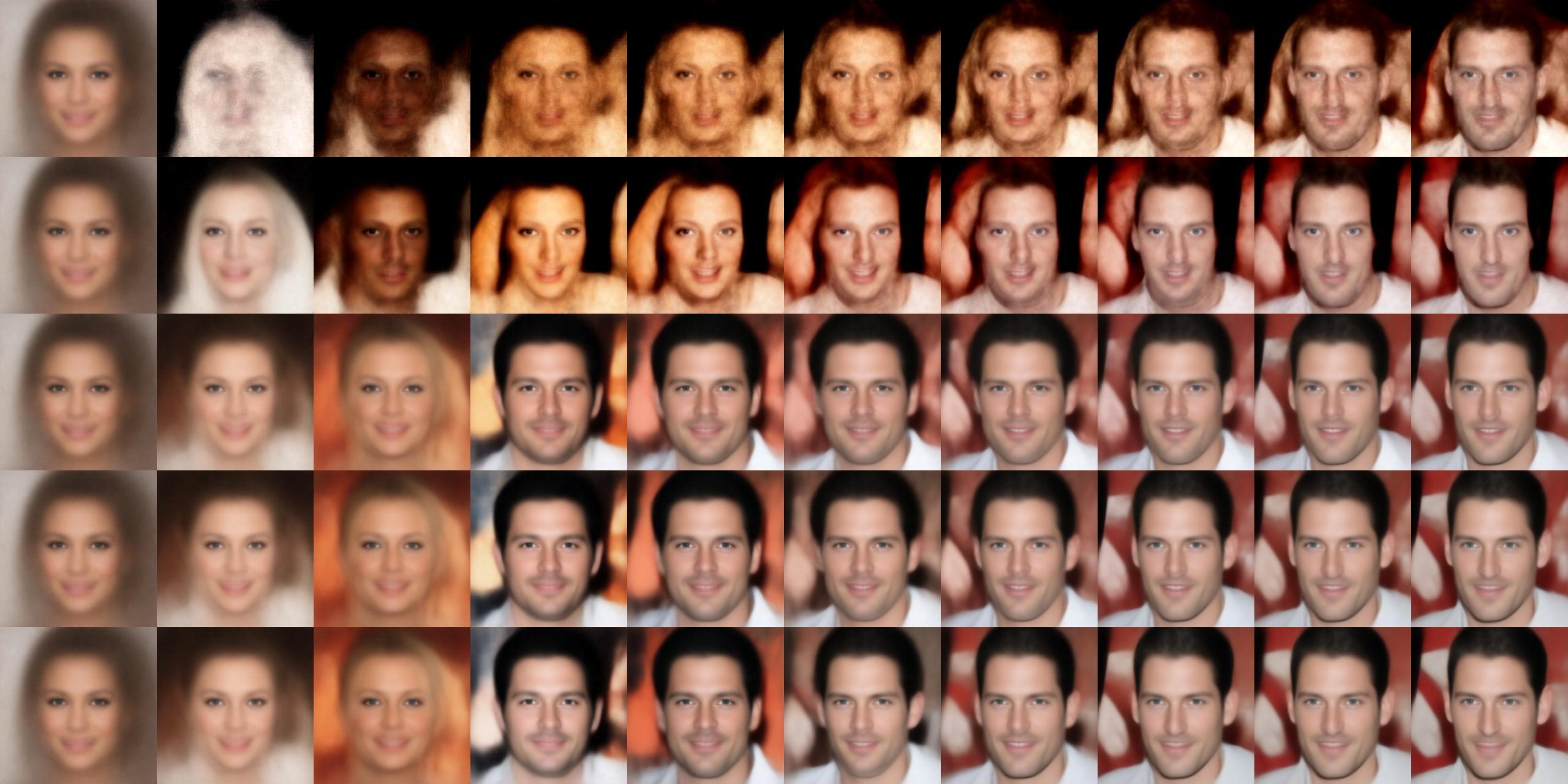}
    \caption{$\widehat x_1$ predictions using \pigdm for the first 10 steps of superresolution showing the convergence of the guided process for $\beta \in [1,0.75,0.5,0.25,0]$ (from top to bottom)}
    \label{fig:placeholder_5}
\end{figure}

\begin{figure}
    \centering
    \includegraphics[width=0.72\linewidth]{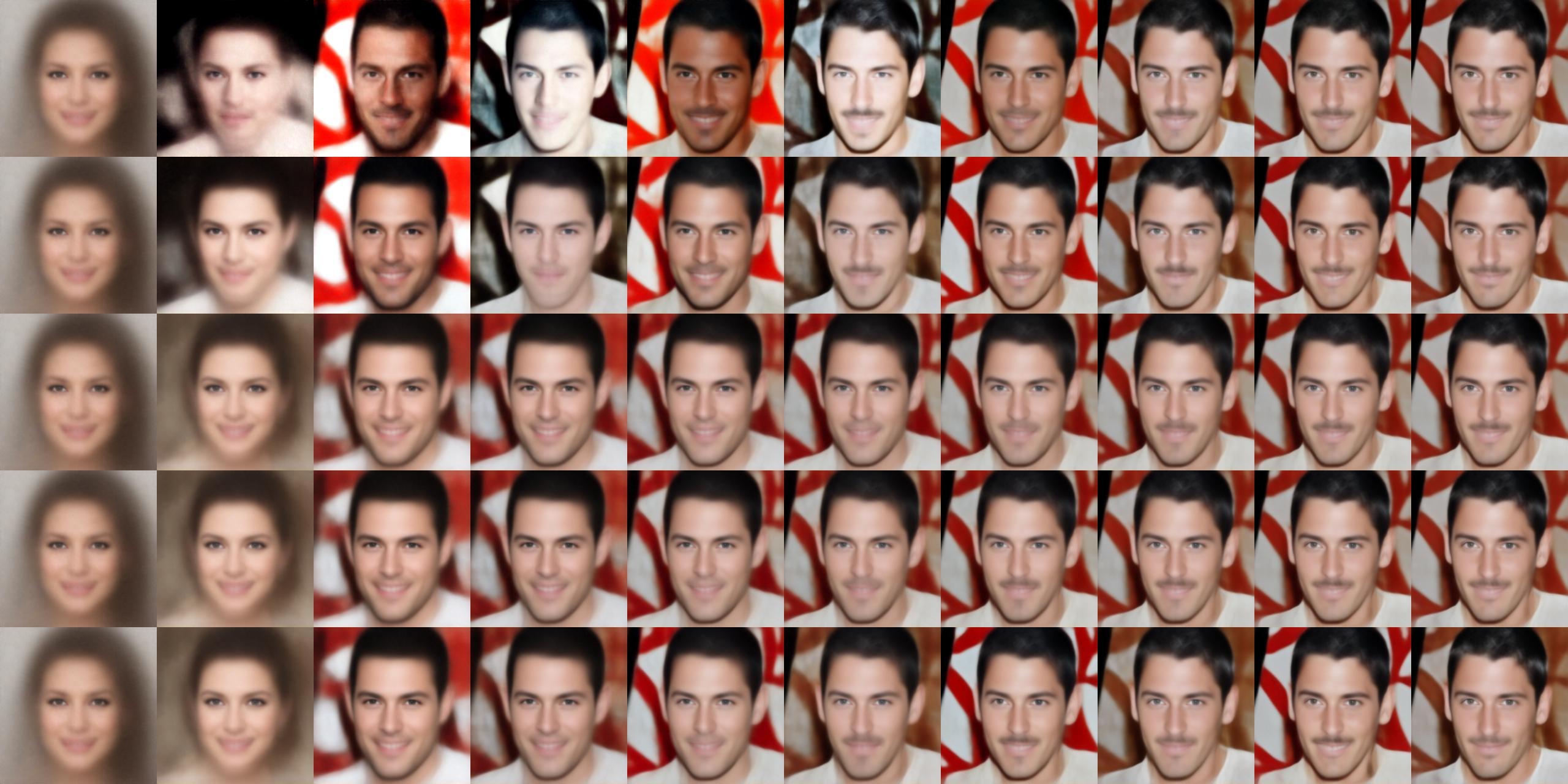}
    \caption{$\widehat x_1$ predictions using \pigdmp for the first 10 steps of deblurring showing the convergence of the guided process for $\beta \in [1,0.75,0.5,0.25,0]$ (from top to bottom)}
    \label{fig:placeholder_6}
\end{figure}

\begin{figure}
    \centering
    \includegraphics[width=0.72\linewidth]{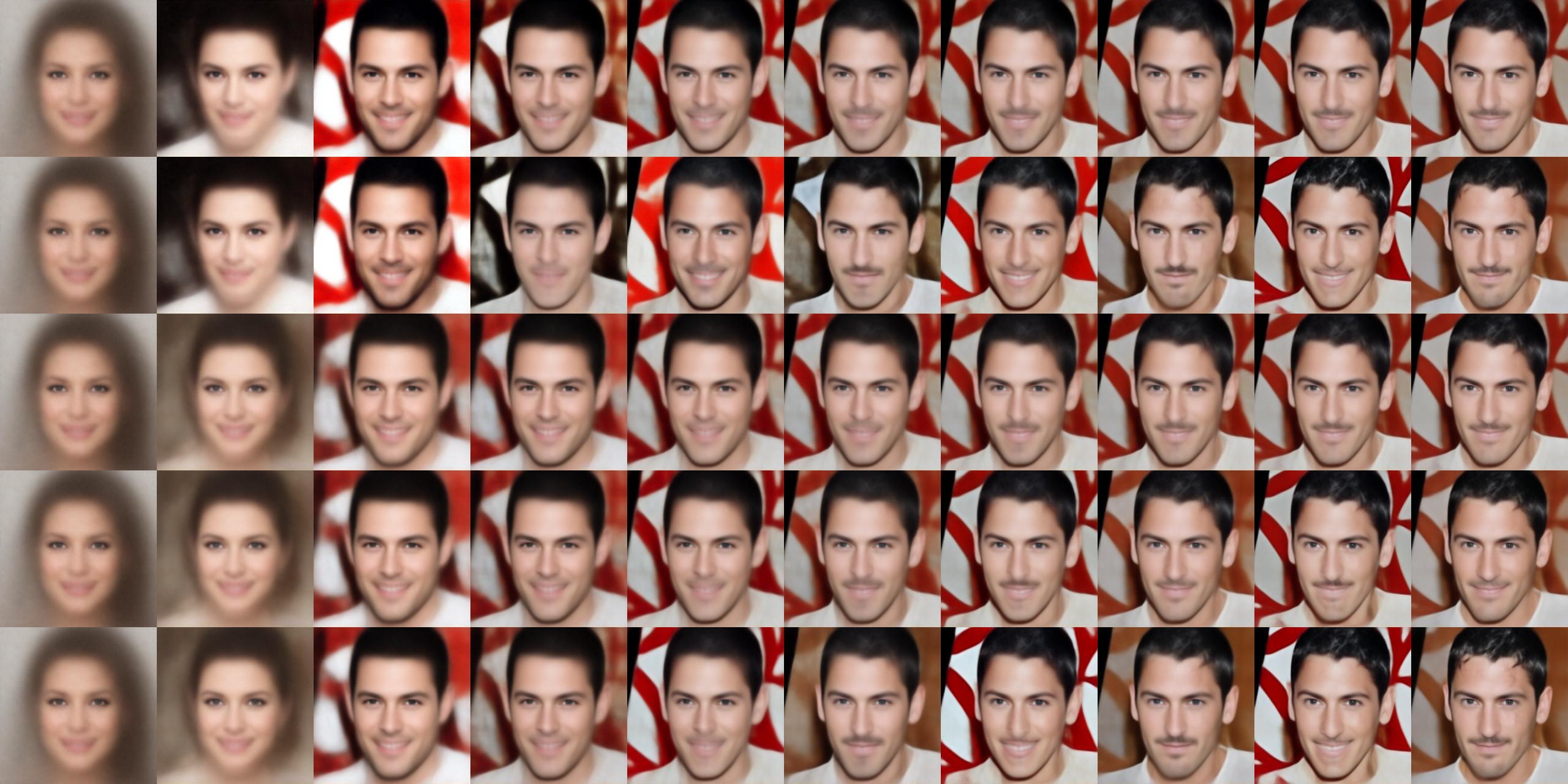}
    \caption{$\widehat x_1$ predictions using \pigdmp for the first 10 steps of inpainting showing the convergence of the guided process for $\beta \in [1,0.75,0.5,0.25,0]$ (from top to bottom)}
    \label{fig:placeholder_7}
\end{figure}

\begin{figure}
    \centering
    \includegraphics[width=0.72\linewidth]{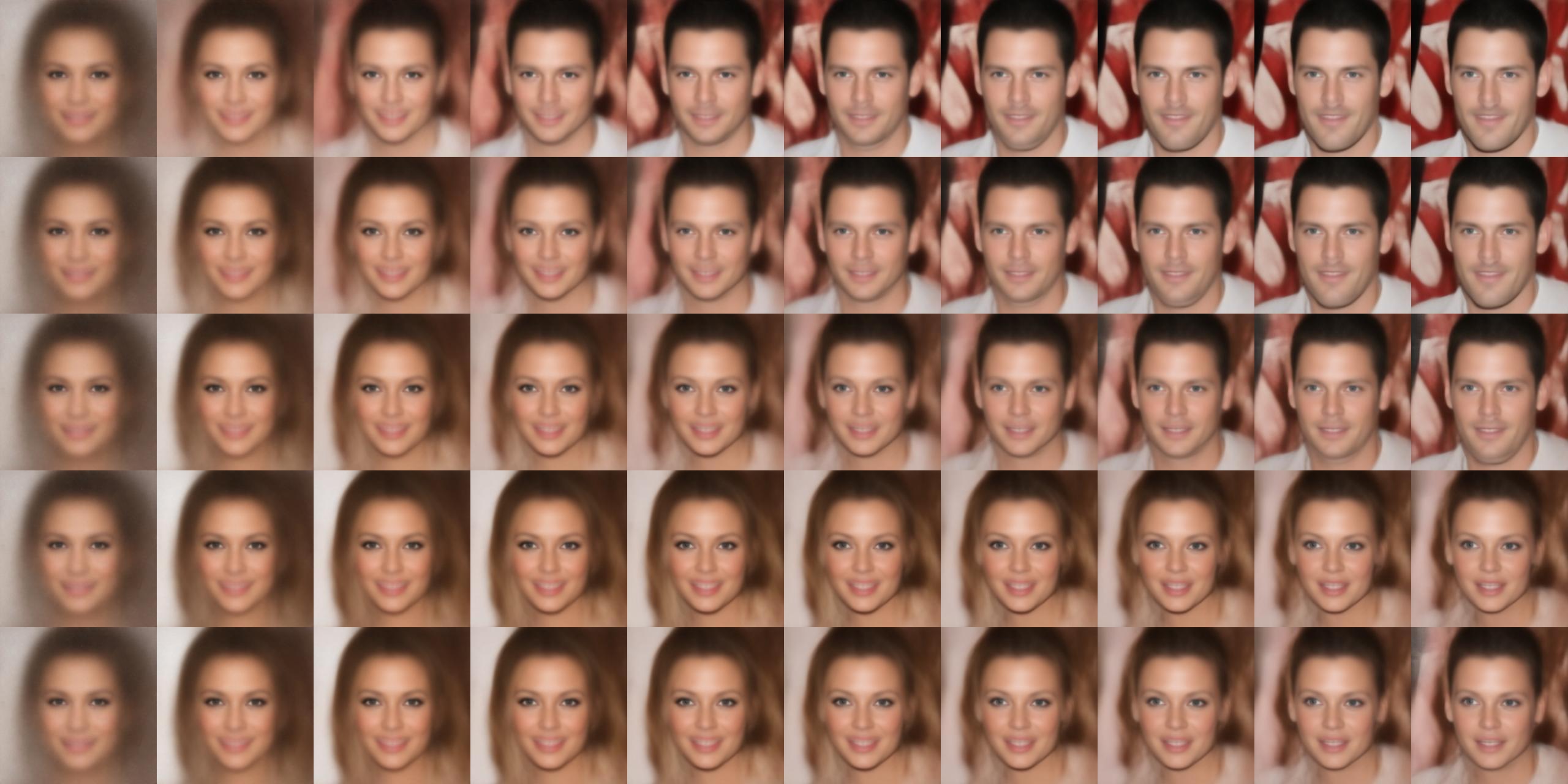}
    \caption{$\widehat x_1$ predictions using \pigdmp for the first 10 steps of superresolution showing the convergence of the guided process for $\beta \in [1,0.75,0.5,0.25,0]$ (from top to bottom)}
    \label{fig:placeholder_8}
\end{figure}
\begin{figure}
    \centering
    \includegraphics[width=0.72\linewidth]{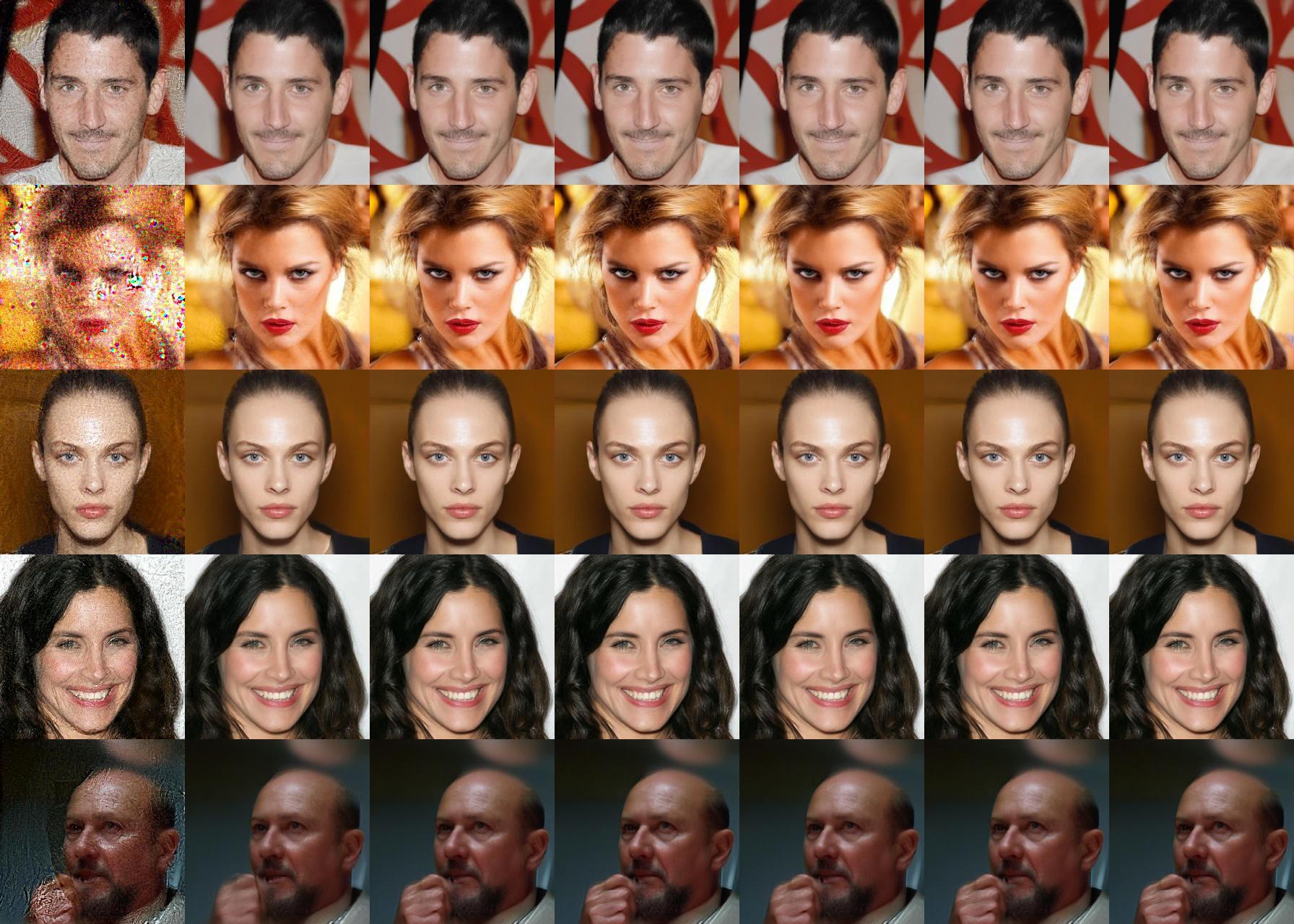}
    \caption{Final samples for deblurring using \pigdm with $\beta \in [1,0.75,0.5,0.25,0]$ (from left to right)}
    \label{fig:placeholder_9}
\end{figure}

\begin{figure}
    \centering
    \includegraphics[width=0.72\linewidth]{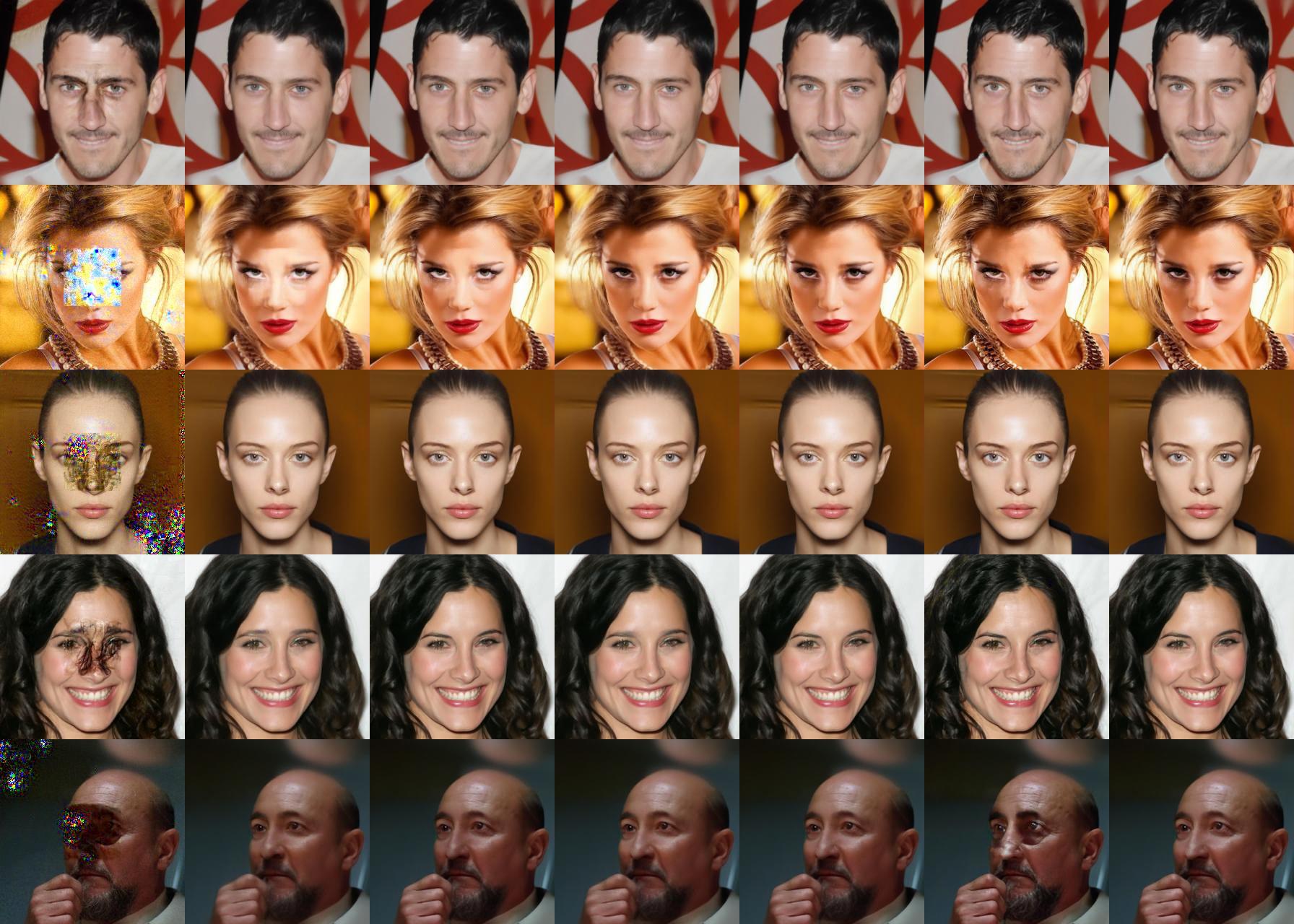}
    \caption{Final samples for inpainting using \pigdm with $\beta \in [1,0.75,0.5,0.25,0]$ (from left to right)}
    \label{fig:placeholder_10}
\end{figure}

\begin{figure}
    \centering
    \includegraphics[width=0.72\linewidth]{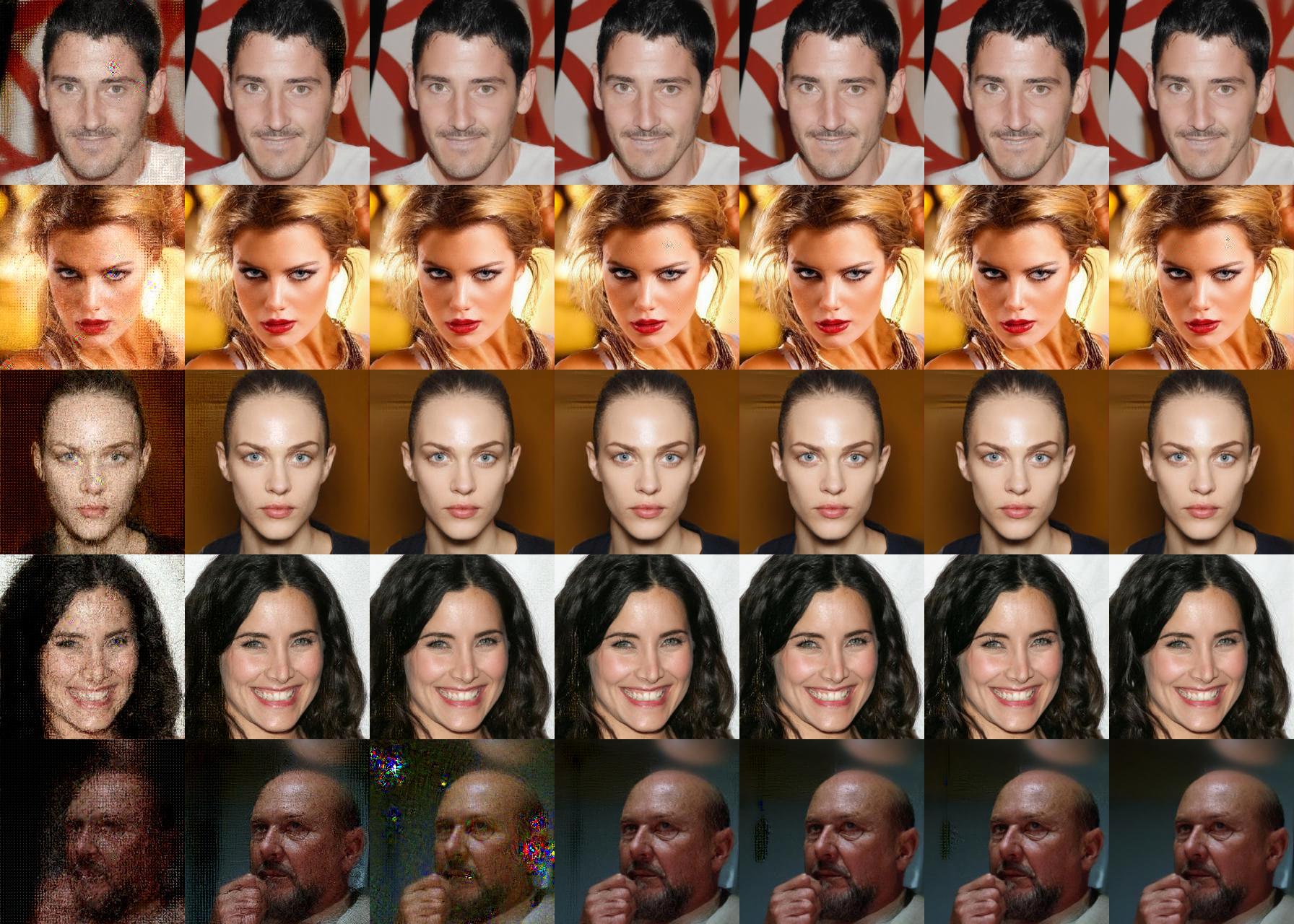}
    \caption{Final samples for superresolution using \pigdm with $\beta \in [1,0.75,0.5,0.25,0]$ (from left to right)}
    \label{fig:placeholder_11}
\end{figure}

\begin{figure}
    \centering
    \includegraphics[width=0.72\linewidth]{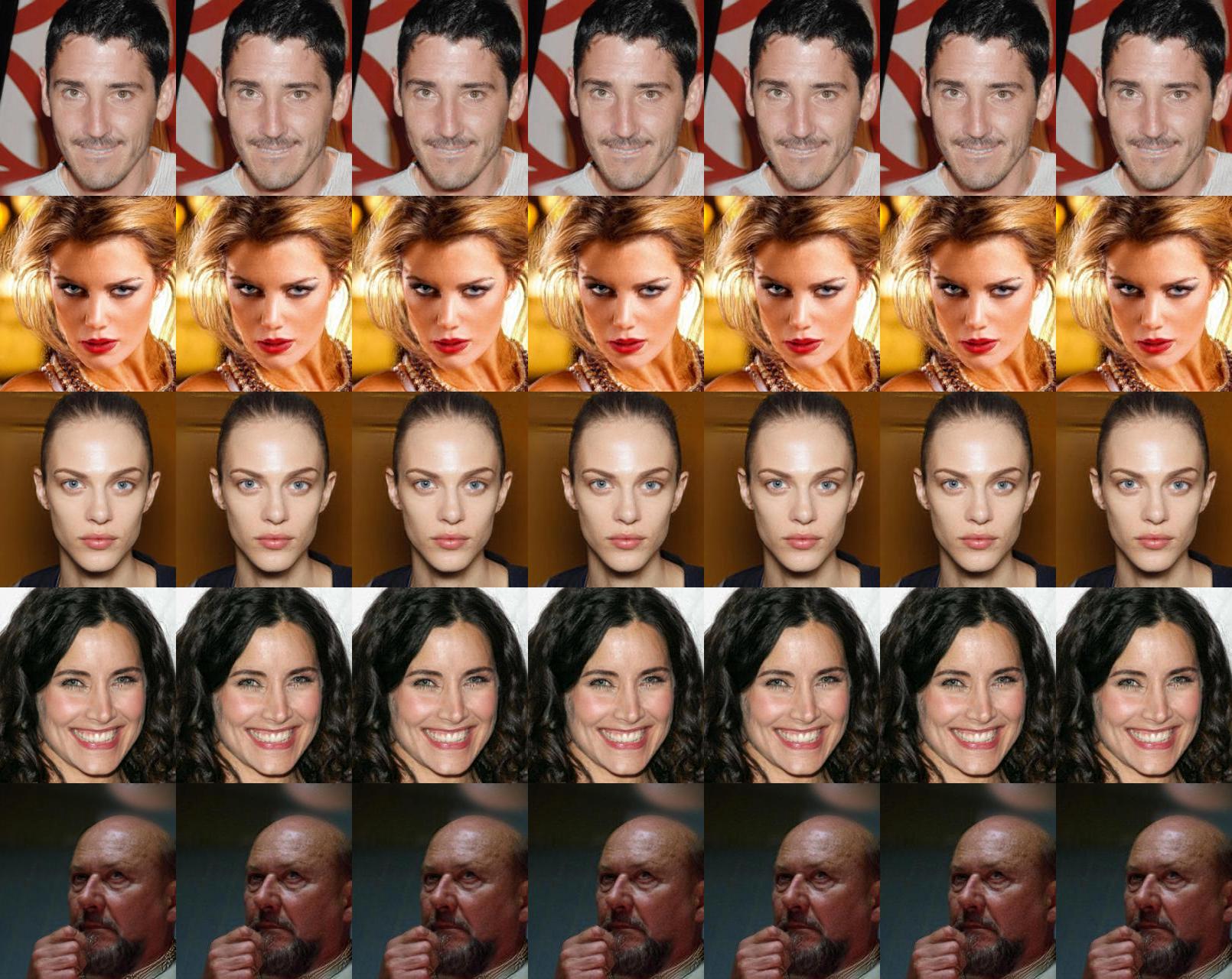}
    \caption{Final samples for deblurring using \pigdmp with $\beta \in [1,0.75,0.5,0.25,0]$ (from left to right)}
    \label{fig:placeholder_12}
\end{figure}

\begin{figure}
    \centering
    \includegraphics[width=0.72\linewidth]{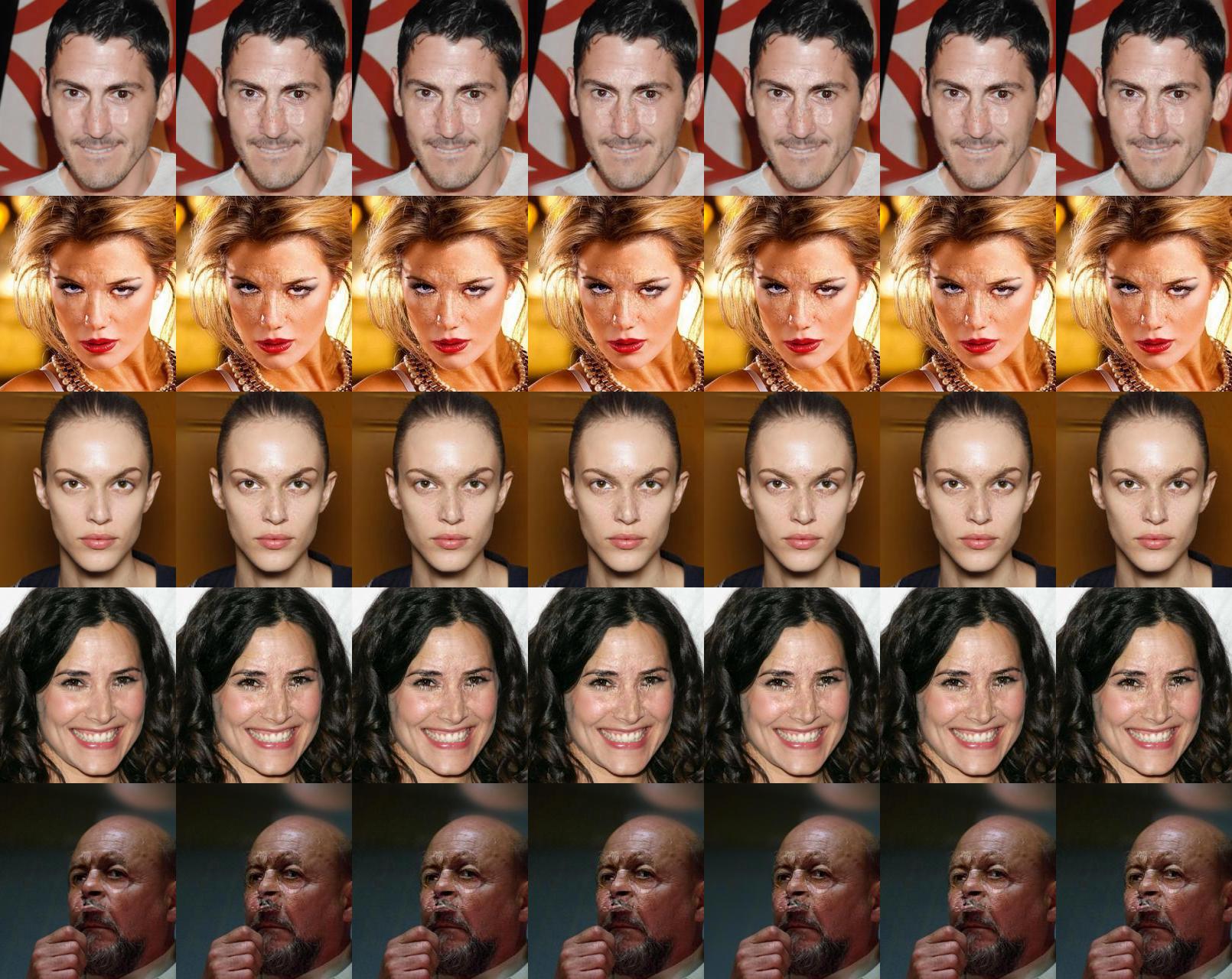}
    \caption{Final samples for inpainting using \pigdmp with $\beta \in [1,0.75,0.5,0.25,0]$ (from left to right)}
    \label{fig:placeholder_13}
\end{figure}

\begin{figure}
    \centering
    \includegraphics[width=0.72\linewidth]{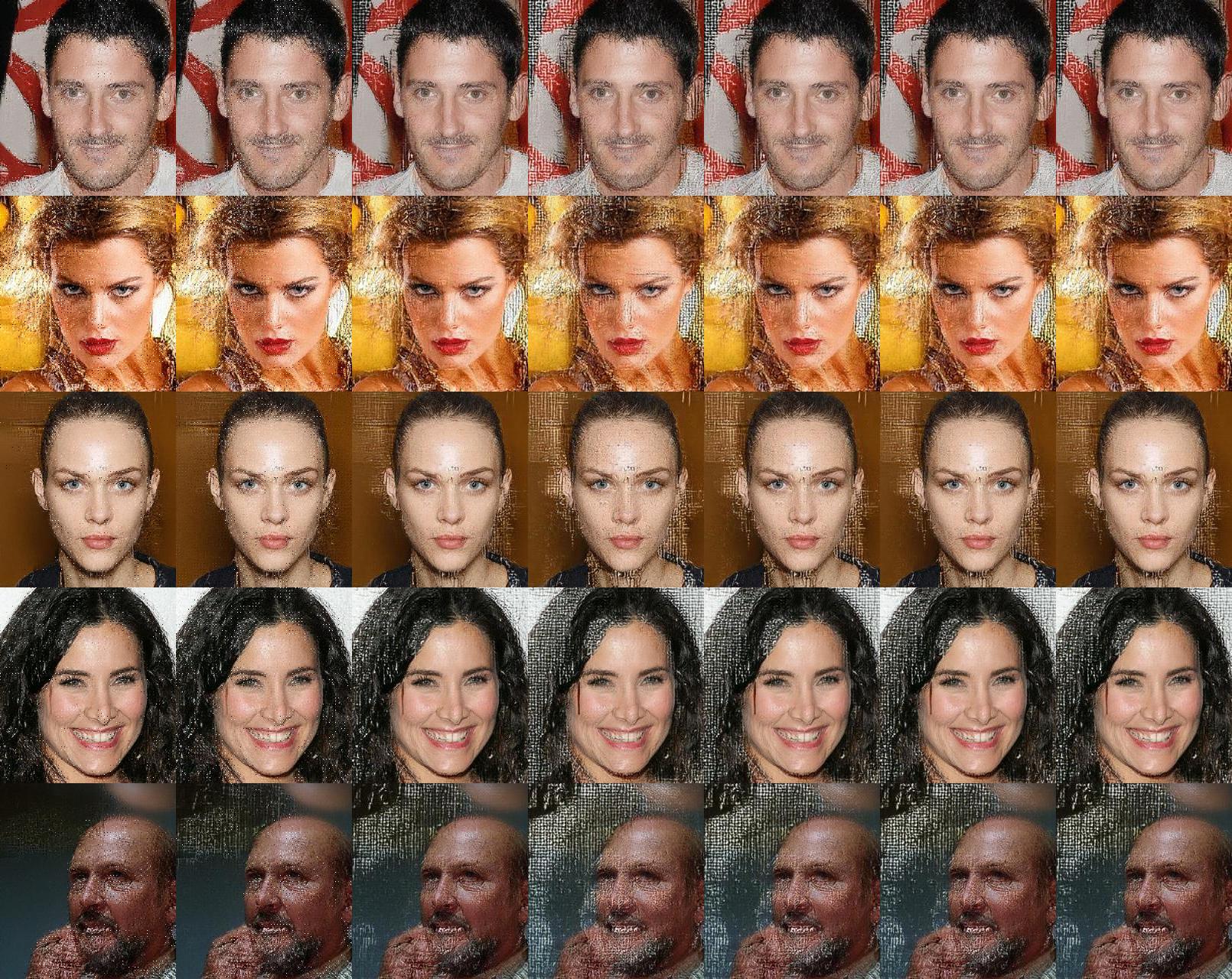}
    \caption{Final samples for superresolution using \pigdmp with $\beta \in [1,0.75,0.5,0.25,0]$ (from left to right)}
    \label{fig:placeholder_14}
\end{figure}

\newpage

\end{document}